\documentclass[final,1p,times,authoryear]{elsarticle}

\usepackage{amssymb}
\usepackage{amsmath}
\usepackage{amsthm}
\usepackage{color}
\usepackage[utf8]{inputenc}
\usepackage{natbib}

\newtheorem{theorem}{Theorem}

\newtheorem{definition}{Definition}
\newtheorem{lemma}{Lemma}
\newtheorem{proposition}{Proposition}
\newtheorem{remark}{Remark}
\newtheorem{assumption}{Assumption}

\newcommand{\valpha}{\boldsymbol{\alpha}}
\newcommand{\vbeta}{\boldsymbol{\beta}}

\newcommand{\vtheta}{\boldsymbol{\theta}}

\newcommand{\vnu}{\boldsymbol{\nu}}

\newcommand{\vb}{\boldsymbol{b}}

\newcommand{\vf}{\boldsymbol{f}}

\newcommand{\vk}{\boldsymbol{k}}

\newcommand{\vm}{\boldsymbol{m}}
\newcommand{\vn}{\boldsymbol{n}}

\newcommand{\vw}{\boldsymbol{w}}
\newcommand{\vx}{\boldsymbol{x}}
\newcommand{\vy}{\boldsymbol{y}}
\newcommand{\vz}{\boldsymbol{z}}
\newcommand{\vA}{\boldsymbol{A}}

\newcommand{\vO}{\boldsymbol{O}}

\newcommand{\vW}{\boldsymbol{W}}

\newcommand{\vZ}{\boldsymbol{Z}}

\newcommand{\sgn}{\mathrm{sgn}}
\newcommand{\I}{\mathrm{i}}

\newcommand{\sK}{\mathbb{K}}

\newcommand{\sN}{\mathbb{N}}

\newcommand{\sR}{\mathbb{R}}

\newcommand{\sT}{\mathbb{T}}

\newcommand{\fB}{\mathcal{B}}
\newcommand{\fC}{\mathcal{C}}
\newcommand{\fD}{\mathcal{D}}

\newcommand{\fF}{\mathcal{F}}
\newcommand{\fG}{\mathcal{G}}
\newcommand{\fH}{\mathcal{H}}

\newcommand{\fJ}{\mathcal{J}}

\newcommand{\fL}{\mathcal{L}}

\newcommand{\fN}{\mathcal{N}}
\newcommand{\fO}{\vO}
\newcommand{\fP}{\mathcal{P}}

\newcommand{\fR}{\mathcal{R}}
\newcommand{\fS}{\mathcal{S}}
\newcommand{\fT}{\mathcal{T}}

\newcommand{\fX}{\mathcal{X}}
\newcommand{\fY}{\mathcal{Y}}
\newcommand{\fZ}{\mathcal{Z}}

\journal{}

\begin{document}

\begin{frontmatter}



\title{DeepONet for Solving Nonlinear Partial Differential Equations with Physics-Informed Training} 


\author[label1]{Yahong Yang} 
\ead{yyang3194@gatech.edu}

\affiliation[label1]{organization={School of Mathematics, Georgia Institute of Technology},
            addressline={686 Cherry Street},
            city={Atlanta},
            postcode={30332},
            state={Georgia},
            country={USA}}


\begin{abstract}
In this paper, we investigate the applications of operator learning, specifically DeepONet, for solving nonlinear partial differential equations (PDEs). Unlike conventional function learning methods that require training separate neural networks for each PDE, operator learning enables generalization across different PDEs without retraining. This study examines the performance of DeepONet in physics-informed training, focusing on two key aspects: (1) the approximation capabilities of deep branch and trunk networks, and (2) the generalization error in Sobolev norms. Our results show that complex branch networks provide substantial performance gains, while trunk networks are most effective when kept relatively simple. Furthermore, we derive a bound on the generalization error of DeepONet for solving nonlinear PDEs by analyzing the Rademacher complexity of its derivatives in terms of pseudo-dimension. This work bridges a critical theoretical gap by delivering rigorous error estimates. This paper fills a theoretical gap by providing {\color{black}error estimates} for a wide range of physics-informed machine learning models and applications.
\end{abstract}




\end{frontmatter}



\section{Introduction}
\label{sec1}
Solving Partial Differential Equations (PDEs) using neural networks has been widely applied in mathematical and engineering fields, especially for high-dimensional domains where classical methods, such as finite elements \citep{brenner2008mathematical}, face challenges. Methods for solving PDEs can be broadly divided into two categories: function learning and operator learning. Function learning methods, such as Physics-Informed Neural Networks (PINNs) \citep{raissi2019physics}, the Deep Ritz Method \citep{weinan2018deep}, the Deep Galerkin Method \citep{sirignano2018dgm}, and approaches based on random features \citep{chen2022bridging, sun2024local, dong2023method}, use neural networks to directly approximate the solutions of PDEs by minimizing specifically designed loss functions. A significant limitation of function learning approaches is the need to train a separate neural network for each PDE, especially if it differs from the one on which the network was originally trained. On the other hand, operator learning methods, such as DeepONet \citep{lu2021deep} and the Fourier Neural Operator (FNO) \cite{li2020fourier}, focus on learning the operator that maps the PDE parameters to their corresponding solutions. These methods are more general and do not require retraining for different PDEs, provided that the underlying differential operator remains unchanged.

Therefore, applying operator learning to solve PDEs is more effective. In this paper, we will study the performance of operator learning in solving PDEs, with a focus on DeepONet \citep{lu2021learning}. In the DeepONet structure introduced by \cite{lu2021learning}, the architecture is expressed as follows:
\begin{equation}\label{classical}
    \sum_{k=1}^p \underbrace{\sum_{i=1}^q c_i^k \sigma\left(\sum_{j=1}^m \xi_{i j}^k s\left(\vx_j\right) + \theta_i^k\right)}_{\text{branch}} 
    \underbrace{\sigma\left(\vw_k \cdot \vy + \zeta_k\right)}_{\text{trunk}},\notag
\end{equation}
where the branch network processes the input function, and the trunk network processes the coordinates.
A key difference between this and shallow neural networks for function approximation is the replacement of the coefficient term with a branch network, which itself resembles a shallow neural network. This facilitates the generalization of the structure into a more flexible form, given by:
\begin{align}
    \fG(s; \vtheta) := \sum_{k=1}^p \underbrace{\fB(\fD(s); \vtheta_{1,k})}_{\text{branch}} \underbrace{\fT(\vy; \vtheta_{2,k})}_{\text{trunk}},
    \label{gen}
\end{align}
where both $\fB$ and $\fT$ are fully connected DNNs. This kind of the structure has also been mentioned in \cite{lu2022comprehensive,li2023phase}. Our approach can easily reduce to the shallow neural network case, recovering the classical DeepONet. In Eq.~\eqref{gen}, $\fD$ denotes a projection that reduces $s$ to a finite-dimensional vector, allowing for various reduction techniques in differential problems, such as the truncation of Fourier and Taylor series or the finite element method \citep{brenner2008mathematical}. In the classical DeepONet, this projection is based on the sample points.

Note that both the branch and trunk networks can easily be generalized into deep neural networks. As discussed in \cite{lu2021deep, yang2023nearly, yang2023nearlys,yang2025deep}, deep neural networks have been shown to outperform shallow neural networks in terms of approximation error for function approximation. This raises the following question:

\textbf{Q1:} Do the branch and trunk networks have sufficient approximation ability when they are deep neural networks? If so, is their performance in terms of approximation error better than that of simple neural networks when they are made deeper and wider?

{\color{black}The results in this paper show that the branch network benefits significantly from increased complexity, whereas the trunk network does not. The reason is that the trunk network serves to represent the basis of the output space, and its approximation ability is constrained by the number of features $p$ rather than by the complexity of each feature representation. In contrast, the branch network approximates the coefficients of these basis functions in an infinite-dimensional space. Approximating this part is inherently challenging, and thus requires a more expressive (deeper and wider) structure. When the branch and trunk networks are balanced, Theorem~\ref{approximation} shows that the complexity requirement of the trunk network only grows logarithmically compared with that of the branch network. Hence, in most cases, the branch network should be complex if we want to improve the approximation error, even though this may increase the training cost. By contrast, increasing the complexity of the trunk network does not improve the approximation error but substantially raises the training complexity.  }

Regarding deep structures, while a deep trunk network may in principle achieve a \textit{super-convergence rate}—a rate better than \( M^{-\tfrac{s}{d}} \), where \( M \) is the number of parameters, \( s \) the smoothness of the function class, and \( d \) the dimension—there are major challenges. First, most of the parameters in the trunk network depend on the input function \( v \). For each such parameter, a branch network must be used to approximate it, making the overall DeepONet structure more complicated than Eq.~\eqref{gen}. Second, this requires the branch network to approximate discontinuous functionals \citep{yarotsky2020phase}, which is highly challenging, since deep neural networks are typically designed to approximate continuous functions \citep{devore1989optimal}. The branch network, however, does not face these issues and can benefit from increased depth in terms of approximation error. This observation is consistent with the findings of \cite{lu2022comprehensive}, which show that the trunk network is sensitive and difficult to train. Therefore, the trunk network should generally be kept simple, with methods such as Proper Orthogonal Decomposition (POD) used to assist in learning.

The second question arises from applying operator learning to solve PDEs. In \cite{lanthaler2022error}, they consider the error analysis of DeepONet in the \( L^2 \)-norm. However, for solving PDEs such as 
\begin{equation}
\begin{cases}
\fL u = f & \text{in } \Omega, \\
u = 0 & \text{on } \partial \Omega,
\end{cases}
\label{PDE}
\end{equation}
where \( \Omega = [0,1]^d \) and \( \fL \) is a second-order differential operator, the loss function should incorporate information about derivatives, similar to the Physics-Informed Neural Network (PINN) loss \citep{raissi2019physics}. Although some papers have applied PINN losses to DeepONet, such as \cite{wang2021learning, lin2023operator, hao2024newton}, a theoretical gap remains in understanding the generalization error of DeepONet when measured in Sobolev norms.

\textbf{Q2:} What is the generalization error of DeepONet in physics-informed training?

In this paper, we address the critical question of generalization error estimation for solving nonlinear PDEs, particularly those that can be expressed in polynomial form involving derivatives of the unknown function and the function itself. Using the Rademacher complexity \cite{anthony1999neural}, we establish bounds on the generalization error by leveraging the concept of pseudo-dimension \cite{pollard1990empirical}.

While previous work such as \cite{gopalani2024towards} has analyzed the Rademacher complexity of DeepONet, they did not account for the derivatives of DeepONet. Additionally, their approach to bounding Rademacher complexity was based on the magnitude of the parameters rather than their size. This parameter-magnitude-based approach is effective when the magnitude is not excessively large. However, in deep neural networks designed to achieve optimal approximation rates, parameter magnitudes can become quite large, as discussed in \cite{siegel2022optimal}.
Our work introduces a novel approach to bounding the Rademacher complexity and generalization error, addressing these limitations. By incorporating the derivatives of DeepONet and focusing on the pseudo-dimension, we provide a more robust theoretical framework that improves upon existing methodologies and fills a significant gap in the analysis of operator learning for nonlinear PDEs.

\subsection{Contributions and Organization of the Paper}
The rest of the paper is organized as follows. In Section 2, we introduce the notations used throughout the paper and set up the operator and loss functions from the problem statement. In Section 3, we estimate the approximation rate. In the first subsection, we reduce operator learning to function learning, showing that the trunk network does not benefit from a deep structure. In the second subsection, we reduce functional approximation to function approximation and demonstrate that the encoding of the input space for the operator can still rely on traditional sampling methods, without the need to incorporate derivative information, even in physics-informed training. In the last subsection, we approximate the function and show that the branch network can benefit from a deep structure. In Section 4, we provide the generalization error analysis of DeepONet in physics-informed training. All omitted proofs from the main text are presented in the appendix.

\subsection{Related works}
\textbf{Function Learning:} Unlike operator learning, function learning focuses on learning a map from a finite-dimensional space to a finite-dimensional space. Many papers have studied the theory of function approximation in Sobolev spaces, such as \cite{mhaskar1996neural, shen2022optimal, siegel2022optimal, yang2023nearly, yang2023nearlys, yang2024near, opschoor2022exponential, yarotsky2017error, yarotsky2020phase}, among others. In \cite{mhaskar1996neural, yang2024near, yarotsky2020phase}, the approximation is considered under the assumption of continuous approximators, as mentioned in \cite{devore1989optimal}. However, for deep neural networks, the approximator is often discontinuous, as discussed in \cite{yarotsky2020phase}, which still achieves the optimal approximation rate with respect to the number of parameters. This is one of the benefits of deep structures, and whether this deep structure provides a significant advantage in approximation is a key point of discussion in this work.

\textbf{Functional Learning:} Functional learning can be viewed as a special case of operator learning, where the objective is to map a function to a constant. Several papers have provided theoretical analyses of using neural networks to approximate functionals. In \cite{chen1993approximations}, the universal approximation of neural networks for approximating continuous functionals is established. In \cite{song2023approximation}, the authors provide the approximation rate for functionals, while in \cite{yang2022approximation}, they address the approximation of functionals without the curse of dimensionality using Barron space methods. In \cite{yang2024sphericalanalysislearningnonlinear}, the approximation rate for functionals on spheres is analyzed.

\textbf{Operator Learning:} The operator learning aspect is closely related to this work. In \cite{lanthaler2022error,chen1995universal,gopalani2024towards,liu2024neural}, an error analysis of DeepONet is performed; however, their analysis is primarily focused on solving PDEs in the \( L^2 \)-norm and does not consider network design, i.e., which parts of the network should be deeper and which should be shallower. In \cite{lanthaler2023operator}, the focus is on addressing the curse of dimensionality and providing lower bounds on the number of parameters required to achieve the target rate. However, in their results, there is a parameter \( \gamma \) that obscures the structural information of the network, making the benefits of network architecture in operator learning unclear. 
Furthermore, for the generalization error analysis, they first assume that \( \mathcal{G}(f; \boldsymbol{\theta})(\boldsymbol{y}) \) is Lipschitz with respect to \( \boldsymbol{\theta} \). However, the Lipschitz constant can be extremely large for deep neural networks \cite{bartlett2019nearly,siegel2022optimal}, making this assumption highly restrictive. Moreover, in their final results, they bound the generalization error using the total number of parameters, which does not effectively distinguish the contributions of different network components, such as the branch net and the trunk net, nor does it explicitly reveal how the depth and width of the neural network affect its performance. This approach treats all factors as having the same influence on the generalization error, which is a rough approximation. 
In \cite{gopalani2024towards}, an analysis of the Rademacher complexity of DeepONet is provided. However, their study does not take into account the derivatives of DeepONet. Additionally, their approach to bounding Rademacher complexity relies on the magnitude of the parameters rather than their size. This parameter-magnitude-based approach is effective when the magnitudes are not excessively large. However, in deep neural networks designed to achieve optimal approximation rates, parameter magnitudes can become significantly large, as discussed in \cite{siegel2022optimal}.

In \cite{liu2022deep,liu2024neural}, they summarize the general theory of operator learning, but the existence of the benefits of deep neural networks is still not explicitly visible in their results. Other works, such as those focusing on Fourier Neural Operators (FNO) \cite{kovachki2021universal}, explore different types of operators, which differ from the focus of this paper.

\section{Preliminaries}
\subsection{Notations}
Let us summarize all basic notations used in the paper as follows:
		
		\textbf{1}. Matrices are denoted by bold uppercase letters. For example, $\vA\in\sR^{m\times n}$ is a real matrix of size $m\times n$ and $\vA^\top$ denotes the transpose of $\vA$.
		
		
		\textbf{2}. For a $d$-dimensional multi-index $\valpha=[\alpha_1,\alpha_2,\cdots\alpha_d]\in\sN^d$, we denote several related notations as follows: $(a)~ |\boldsymbol{\alpha}|=\left|\alpha_1\right|+\left|\alpha_2\right|+\cdots+\left|\alpha_d\right|$; $(b)~\boldsymbol{x}^\alpha=x_1^{\alpha_1} x_2^{\alpha_2} \cdots x_d^{\alpha_d},~ \boldsymbol{x}=\left[x_1, x_2, \cdots, x_d\right]^\top$; $ (c)~\boldsymbol{\alpha} !=\alpha_{1} ! \alpha_{2} ! \cdots \alpha_{d} !.$
		
		\textbf{3}. Let $B_{r,|\cdot|}(\vx)\subset\sR^d$ be the closed ball with a center $\vx\in\sR^d$ and a radius $r$ measured by the Euclidean distance. Similarly, $B_{r,\|\cdot\|_{\ell_\infty}}(\vx)\subset\sR^d$ be the closed ball with a center $\vx\in\sR^d$ and a radius $r$ measured by the $\ell_\infty$-norm.
		
		\textbf{4}. Assume $\vn\in\sN_+^n$, then $f(\vn)=\vO(g(\vn))$ means that there exists positive $C$ independent of $\vn,f,g$ such that $f(\vn)\le Cg(\vn)$ when all entries of $\vn$ go to $+\infty$.
		
		\textbf{5}. Define $\sigma_1(x):=\sigma(x)=\max\{0,x\}$ and $\sigma_2:=\sigma^2(x)$. We call the neural networks with activation function $\sigma_t$ with $t\le i$ as $\sigma_i$ neural networks ($\sigma_i$-NNs). With the abuse of notations, we define $\sigma_i:\sR^d\to\sR^d$ as $\sigma_i(\vx)=\left[\begin{array}{c}
			\sigma_i(x_1) \\
			\vdots \\
			 \sigma_i(x_d)
		\end{array}\right]$ for any $\vx=\left[x_1, \cdots, x_d\right]^T \in\sR^d$.
		
		\textbf{6}. Define $L,W\in\sN_+$, $W_0=d$ and $W_{L+1}=1$, $W_i\in\sN_+$ for $i=1,2,\ldots,L$, then a $\sigma_i$-NN $\phi$ with the width $W$ and depth $L$ can be described as follows:\[\boldsymbol{x}=\tilde{\boldsymbol{h}}_0 \stackrel{W_1, b_1}{\longrightarrow} \boldsymbol{h}_1 \stackrel{\sigma_i}{\longrightarrow} \tilde{\boldsymbol{h}}_1 \ldots  \tilde{\boldsymbol{h}}_L \stackrel{W_{L+1}, b_{L+1}}{\longrightarrow} \phi(\boldsymbol{x})=\boldsymbol{h}_{L+1},\] where $\vW_i\in\sR^{W_i\times W_{i-1}}$ and $\vb_i\in\sR^{W_i}$ are the weight matrix and the bias vector in the $i$-th linear transform in $\phi$, respectively, i.e., $\boldsymbol{h}_i:=\boldsymbol{W}_i \tilde{\boldsymbol{h}}_{i-1}+\boldsymbol{b}_i, ~\text { for } i=1, \ldots, L+1$ and $\tilde{\boldsymbol{h}}_i=\sigma_i\left(\boldsymbol{h}_i\right),\text{ for }i=1, \ldots, L.$ In this paper, an DNN with the width $W$ and depth $L$, means
		(a) The maximum width of this DNN for all hidden layers less than or equal to $W$.
		(b) The number of hidden layers of this DNN less than or equal to $L$.

  \textbf{7}.
			Denote $\Omega$ as $[0,1]^d$, $D$ as the weak derivative of a single variable function and $D^{\valpha}=D^{\alpha_1}_1D^{\alpha_2}_2\ldots D^{\alpha_d}_d$ as the partial derivative where $\valpha=[\alpha_{1},\alpha_{2},\ldots,\alpha_d]^T$ and $D_i$ is the derivative in the $i$-th variable. Let $n\in\sN$ and $1\le p\le \infty$. Then we define Sobolev spaces\[W^{n, p}(\Omega):=\left\{f \in L^p(\Omega): D^{\valpha} f \in L^p(\Omega) \text { with }|\boldsymbol{\alpha}| \leq n\right\}\] with a norm \[\|f\|_{W^{n, p}(\Omega)}:=\left(\sum_{0 \leq|\alpha| \leq n}\left\|D^{\valpha} f\right\|_{L^p(\Omega)}^p\right)^{1 / p},~p<\infty,\] and $\|f\|_{W^{n, \infty}(\Omega)}:=\max_{0 \leq|\alpha| \leq n}\left\|D^{\valpha} f\right\|_{L^\infty(\Omega)}$. For simplicity, $\|f\|_{W^{n, 2}(\Omega)}=\|f\|_{H^{n}(\Omega)}$.
			Furthermore, for $\vf=(f_1,f_2,\ldots,f_d)$, $\vf\in W^{1,\infty}(\Omega,\sR^d)$ if and only if $ f_i\in W^{1,\infty}(\Omega)$ for each $i=1,2,\ldots,d$ and \[\|\vf\|_{W^{1,\infty}(\Omega,\sR^d)}:=\max_{i=1,\ldots,d}\{\|f_i\|_{W^{1,\infty}(\Omega)}\}.\]

\subsection{Error component in physics-informed training}
In this paper, we consider applying DeepONet to solve partial differential equations (PDEs). Without loss of generality, we focus on the following PDEs \eqref{PDE}, though our results can be easily generalized to more complex cases. In this paper, we assume that \( \fL \) satisfies the following conditions:

\begin{assumption}\label{asspde}
    For any \( f_1,f_2 \in \fX \subset W^{n,\infty}(\Omega) \) and \( u_1,u_2 \in \fY \subset W^{n,\infty}(\Omega) \) for \( n \geq 2 \), which is the solution of Eq.~(\ref{PDE}) with source term \( f_1,f_2 \), the following conditions hold:
    
    (i) There exists a constant \( C \) such that
    \[
    \|u_1-u_2\|_{H^2(\Omega)} \leq C \|f_1-f_2\|_{L^2(\Omega)}.
    \]

    (ii) There exists a constant \( L \) such that
    \[
    \|u_1-u_2\|_{W^{n,\infty}(\Omega)} \leq L \|f_1-f_2\|_{W^{n,\infty}(\Omega)}.
    \]
    
    (iii) There exists a constant \( P \) such that
    \[
    \max\{\|v_1-v_2\|_{L^{2}(\partial\Omega)},\|\fL v_1-\fL v_2\|_{L^{2}(\Omega)} \}\leq  P\|v_1-v_2\|_{H^{2}(\Omega)},
    \]where $v\in H^{2}(\Omega)\cap C(\bar{\Omega})$.\end{assumption}

{\color{black}The above assumptions are satisfied in many practical settings.  
For example, \cite{grisvard2011elliptic} shows that Condition~(i) in 
Assumption~\ref{asspde} holds when $\mathcal{L}$ is a linear elliptic operator 
with smooth coefficients and the domain $\Omega$ is a polygon. Moreover, the same 
assumption holds for convex domains provided that the coefficients satisfy the 
Cordes condition~\cite{smears2013discontinuous}, and for $C^{1,1}$ domains provided 
that the coefficients belong to VMO~\cite{chiarenza1993w2p}. In all these cases, 
the constant $C$ depends only on $\Omega$ and $\mathcal{L}$.

Condition (ii) in Assumption~\ref{asspde} requires the target operator to be Lipschitz continuous in
\( W^{n,\infty}(\Omega) \).
If \( \Omega \) is smooth, \( \mathcal{L} \) is linear with smooth coefficients, and
\( f \in C^{n}(\Omega) \), classical elliptic regularity yields
\[
\|u\|_{W^{n+2,\infty}(\Omega)}
   \le \|u\|_{C^{n+2}(\Omega)}
   \le L\,\|f\|_{C^{n}(\Omega)}=L\,\|f\|_{W^{n,\infty}(\Omega)},
\]
see \citet{evans2022partial}.  
This inequality is stronger than the condition imposed here. The constant $L$ in this case also only depends on $\Omega$ and $\fL$.

Condition (iii) Assumption~\ref{asspde} has two components.  
The first follows from the trace theorem: if \( \Omega \) is bounded and
\( \partial\Omega \) is \( C^{1} \) (or even Lipschitz;
see \citet{adams2003sobolev}), the requirement is met.  
The second concerns \( \mathcal{L} \); it is automatically satisfied for linear
operators with bounded, smooth coefficients. The constant $P$ in this case also only depends on $\Omega$ and $\fL$.
}

 Based on above assumption, the equation above has a unique solution, establishing a mapping between the source term $f\in\fX$ and the solution $u\in\fY$, denoted as $\fG_*$. In \cite{lu2021learning}, training such an operator is treated as a black-box approach, i.e., a data-driven method. However, in many cases, there is insufficient data for training, or there may be no data at all. Thus, it becomes crucial to incorporate information from the PDEs into the training process, as demonstrated in \cite{wang2021learning,lin2023operator,hao2024newton}. This technique is referred to as physics-informed training.

In other words, we aim to use \( \fG(~\cdot~;\vtheta) \) (see Eq.~(\ref{gen})) to approximate the mapping between \( f \) and \( u \). The chosen loss function is:
\begin{align}
   L_D(\vtheta)&:= L_\text{in}(\vtheta) + \beta L_\text{bound}(\vtheta), \notag \\
   L_\text{in}(\vtheta)&:= \int_{\fX}\int_\Omega  |\fL \fG(f;\vtheta)(\vy) - f(\vy)|^2 \, \mathrm{d} \vy \, \mathrm{d} \mu_{\fX}, \notag \\
   L_\text{bound}(\vtheta)&:= \int_{\fX} \int_{\partial \Omega} | \fG(f;\vtheta)(\vy)|^2 \, \mathrm{d} \vy \, \mathrm{d} \mu_{\fX},\notag
\end{align}
where \( \mu_{\fX} \) is a measure on \( \fX \), and \( \fX \) is the domain of the operator, i.e., \( f \in \fX \). The constant \( \beta \) balances the boundary and interior terms. The error analysis of \( L_\text{bound}(\vtheta) \) is the classical \( L^2 \) estimation. Due to {\color{black}Condition (iii) in Assumption \ref{asspde}}, we obtain:
\begin{align}
    L_D(\vtheta) \le C \int_{\fX} \|\fG(f;\vtheta)(\vy) - \fG_*(f)(\vy)\|^2_{H^2(\Omega)} \, \mathrm{d} \mu_{\fX}.\notag
\end{align}
Furthermore, if we consider that \( \mu_{\fX} \) is a finite measure on \( \fX \) and, for simplicity, assume \( \int_{\fX} \mathrm{d} \mu_{\fX} = 1 \), we have:
\begin{align}
    L_D(\vtheta) \leq C \sup_{f \in \fX} \|\fG(f;\vtheta)(\vy) - \fG_*(f)(\vy)\|^2_{H^2(\Omega)}.\notag
\end{align}
This error is called the approximation error. The other part of the error comes from the sample, which we refer to as the generalization error. To define this error, we first introduce the following loss, which is the discrete version of \( L_D(\vtheta) \):
\begin{equation}
    L_S(\vtheta) := \frac{1}{M} \sum_{i=1}^M \left[\frac{1}{P}\sum_{j=1}^P |\fL \fG(f_i;\vtheta)(\vy_j) - f_i(\vy_j)|^2+\frac{\beta}{Q}\sum_{k=1}^Q | \fG(f_i;\vtheta)(\vz_k)|^2\right],\notag
\end{equation}
where \( \vy_j \) is an i.i.d. uniform sample in \( \Omega \), \( \vz_j \) is an i.i.d. uniform sample in \( \partial\Omega \) and \( f_i \) is an i.i.d. sample based on \( \mu_{\fX} \).

Then, the generalization error is given by:
\begin{align}
    \mathbb{E} L_D(\vtheta_S) \leq L_D(\vtheta_D) + \mathbb{E} \left[ L_D(\vtheta_S) - L_S(\vtheta_S) \right], \label{total}
\end{align}
where \( \vtheta_D = \arg\min L_D(\vtheta) \) and \( \vtheta_S = \arg\min L_S(\vtheta) \). The expectation symbol is due to the randomness in sampling \( \vy \) and \( f \). The generalization error is represented by the last term in Eq.~(\ref{total}).

\section{Approximation Error}

In this section, we estimate the approximation error, i.e., we aim to establish the appropriate DeepONet structure and bound \( \inf_{\vtheta} L_D(\vtheta) \). The proof sketch is divided into three steps: the first step reduces the operator learning problem to a functional learning problem, the second step reduces functional learning to function learning, and the final step approximates the function. The diagram outlining the proof is shown in Figure~\ref{sketch app}.
\begin{figure}[h]
\centering
\vspace{0.1in}
\includegraphics[scale=0.57]{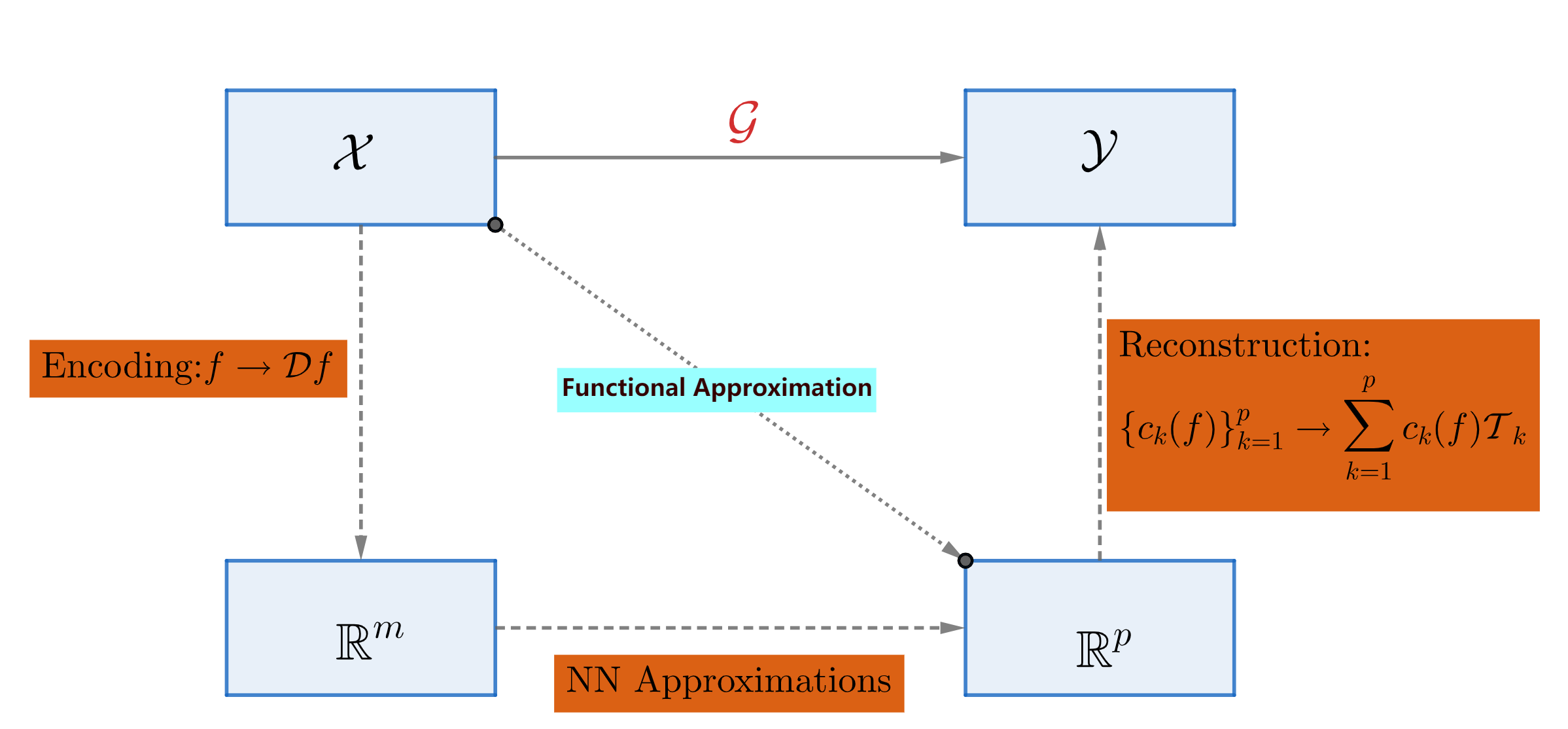}
\vspace{0.1in}
\caption{Sketch of the proof for the approximation error.}
\label{sketch app}
\end{figure}

\begin{theorem}\label{approximation}
    For any \( p, m \in \mathbb{N}^+ \) and sufficiently large \( q \), \( \lambda \in [1, 2] \), and \( \fG_*: \fX \subset H^{s}(\sT^d) \to H^{s}(\sT^d) \) with \( s > n + \frac{d}{2} \), where \( \max\{\|f\|_{W^{n,\infty}(\Omega)},\|f\|_{H^{s}(\Omega)}\} \le M \) for any \( f \in \fX \) and \( M > 0 \), satisfying Assumption \ref{asspde}, there exist \( \sigma_2 \)-NNs \( \fT(\vy; \vtheta_{2,k}) \) with depth \( 3 + \log_2 d - 1 + \log_2 n - 1 \), width \( 4n - 4 + 6d \), and a map \( \fD \) from \( \fX \to [-M, M]^m \), as well as $\sigma_1$-NNs \( \fB(\fD f; \vtheta_k) \) with \( C_1 9^mq \) parameters such that:
\begin{align}
    \sup_{f \in \fX} \left\| \fG_*(f) - \sum_{k=1}^{p} \fB(\fD f; \vtheta_k) \fT(\vy; \vtheta_{2,k}) \right\|_{H^{2}(\Omega)}
   \le C \left( p^{-\frac{n-2}{d}} + p^{\frac{2}{d}} m^{-\frac{s-s'}{d}}  +  m^{\frac{s'}{d}}p^{\frac{2}{d}}q^{-\frac{\lambda}{m}}\right),\notag
\end{align}
where \( s' \in \left(n + \frac{d}{2}, s\right) \), and \( C, C_1 \) are independent of \( m,p \) and \( q \). Furthermore, for \( \lambda = 1 \), \( \fB(\vz; \vtheta_k) \) is a shallow neural network that can achieve this approximation rate. As \( \lambda \) approaches 2, the ratio between the width and depth of \( \fB \) decreases, implying that the network structure becomes deeper.
\end{theorem}
\begin{remark}
    The term \( \lambda \) incorporates the benefit of the deep structure of the neural network. This deep structure arises from the branch network rather than the trunk network. In the approximation analysis, the deep structure of the trunk network does not yield benefits; instead, it makes the branch network more difficult to train. We will provide more details on this in the following subsections.

    {\color{black}Following \citet{liu2024neural}, we balance the terms in Theorem~\ref{approximation} as follows: set
\( q = \exp\bigl(\tfrac{s}{d\lambda}\, m\ln m\bigr) \).
With this choice, the last two terms are
\(\vO\bigl[p^{2/d}\,(\tfrac{\ln\!\ln q}{\lambda\ln q})^{(s-s')/d}\bigr]\).
Taking
\( p = \vO\bigl[(\tfrac{\ln\!\ln q}{\lambda\ln q})^{-(s-s')/n}\bigr] \)
yields the overall error bound
\(\vO\bigl[(\tfrac{\ln\!\ln q}{\lambda\ln q})^{\frac{s-s'}{d}(1-\frac{2}{n})}\bigr] \).
Thus, as \(\lambda\) increases, the total error decreases and the balanced values of \(p\) and \(m\) become smaller.
}
\end{remark}
We applied $\sigma_2$ because the trunk network requires a smoother approximation in $W^{2,\infty}$, whereas the branch network only needs an approximation in $W^{1,\infty}$. Thus, $\sigma_1$ is sufficient for the branch network. However, $\sigma_2$ can still be used for the branch network approximation without any significant difference. This is because the complexity of the two networks with respect to depth is the same, and the approximation abilities of $\sigma_1$ and $\sigma_2$ are equivalent when the depth is of the same order, as shown in \cite{bartlett2019nearly}.

\subsection{Operator learning to functional learning in Sobolev spaces}

\begin{proposition}
\label{operator}
For any \( p \in \mathbb{N} \) and \( v \in \fY \subset W^{n,\infty}(\Omega) \), there exist \( p \) \( \sigma_2 \)-NNs \( \{\fT(\vy;\vtheta_{2,k})\}_{k=1}^p \) and bounded linear functionals \( c_k: W^{n-1,\infty}(\Omega) \to \mathbb{R} \) such that
\begin{equation}
	\left\|v - \sum_{k=1}^{p} c_{k}(v) \fT(\vy; \vtheta_{2,k})\right\|_{H^{2}(\Omega)} \leq C p^{-\frac{n-2}{d}} \|v\|_{H^{n}(\Omega)},\notag
\end{equation}
where \( \fT(\vy; \vtheta_{2,k}) \) is a \( \sigma_2 \)-NN with depth \( 3 + \log_2 d - 1 + \log_2 n - 1 \) and width \( 4n - 4 + 6d \), and \( C \) is a constant$\footnote{For simplicity of notation, in the following statements, $C$ may represent different values from line to line.}$ independent of \( p \) and \( v \). Furthermore, $\sum_{k=1}^p\|\fT(\vy; \vtheta_{2,k})\|^2_{H^{2}(\Omega)}=\fO(p^{\frac{4}{d}})$.
\end{proposition}
\begin{remark}
Before we present the sketch of the proof for Proposition \ref{operator}, there are some remarks we would like to mention. First, the number of parameters in each \( \fT(\vy; \vtheta_{2,k}) \) is independent of \( p \) and \( f \). This is a generalization of the trunk network in DeepONet. It can easily reduce to the classical DeepONet case by applying a shallow neural network to approximate the trunk, a well-known result that we omit here. 

Second, the target function \( v \) can be considered in \( W^{n,p}(\Omega) \) for a more general case. The method of proof remains the same; the only difference is that we need to assume that for any open set \( \Omega_* \subset \Omega \), the following inequality holds:
\[
\|v\|_{W^{n,p}(\Omega_*)}|\Omega_*|^{-\frac{1}{p}} \leq C\|v\|_{W^{n,p}(\Omega)},
\]
for any \( v \in \fY \). It is straightforward to verify that this condition holds for \( p = \infty \).
\end{remark}

The sketch of the proof of Proposition \ref{operator} is as follows. First, we divide \( \Omega \) into \( p \) parts and apply the Bramble--Hilbert Lemma \citep[Lemma 4.3.8]{brenner2008mathematical} to locally approximate \( v \) on each part using polynomials. The coefficients of these polynomials can be regarded as continuous linear functionals in \( W^{n,\infty}(\Omega) \). Next, we use a partition of unity to combine the local polynomial approximations. Finally, we represent both the polynomials and the partition of unity using neural networks.

Based on Proposition \ref{operator}, we can reduce the operator learning problem to functional learning as follows:
\begin{align}
    &\left\| \fG_*(f)- \sum_{k=1}^{p} c_{k}(\fG_*(f)) \fT(\vy; \vtheta_{2,k})\right\|_{H^{2}(\Omega)}\le C p^{-\frac{n-2}{d}} \|\fG_*(f)\|_{H^{n}(\Omega)}\notag\\\le & C p^{-\frac{n-2}{d}} \|\fG_*(f)\|_{W^{n,\infty}(\Omega)}\le  C p^{-\frac{n-2}{d}} \|f\|_{W^{n,\infty}(\Omega)},\notag
\end{align}
where the last inequality follows from Assumption \ref{asspde}. Therefore, without considering the approximation of \( c_{k}(\fG_*(f)) \), the approximation error can be bounded by \( \sup_{f \in \fX} C p^{-\frac{n-2}{d}} \|f\|_{W^{n,\infty}(\Omega)} \). The next step is to consider the approximation of the functional \( c_{k}(\fG_*(f)) \).

The approximation order achieved in Proposition \ref{operator} corresponds to the optimal approximation rate for neural networks in the continuous approximation case. If we consider deep neural networks, we can achieve an approximation rate of \( p^{-\frac{2(n-2)}{d}} \) using the neural network \( \fT(\vy; \vtheta_2(v)) \). As discussed in \cite{yang2023nearlys}, most of the parameters depend on \( v \), leading to significant interaction between the trunk and branch networks, which makes the learning process more challenging. Furthermore, this approximator can be a discontinuous one, as highlighted in \citep{yarotsky2020phase, yang2024near}, where the functionals \( c_k \) may no longer be continuous, thereby invalidating subsequent proofs. Therefore, while it is possible to construct deep neural networks for trunk networks, we do not gain the benefits associated with deep networks, and the approximation process becomes more difficult to learn.

\subsection{Functional learning to function learning in Sobolev spaces}
In this part, we reduce each functional \( c_{k}(\fG_*(f)) \) to a function learning problem by applying a suitable projection \( \fD \) in Eq.~(\ref{gen}) to map \( u \) into a finite-dimensional space. To ensure that the error in this step remains small, we require \( f - \fP \circ \fD f \) to be small for a continuous reconstruction operator \( \fP \), which will be defined later, as shown below:
\begin{align}
    &|c_{k}(\fG_*(f)) - c_{k}(\fG_*(\fP \circ \fD f))|  \le C \|\fG_*(f) - \fG_*(\fP \circ \fD f)\|_{W^{n,\infty}(\Omega)} \notag \\
    \leq &C \|f - \fP \circ \fD f\|_{W^{n,\infty}(\Omega)}.\notag
\end{align}
The first inequality holds because \( c_k \) is a bounded linear functional in \( W^{n-1,\infty}(\Omega) \), and thus also a bounded linear functional in \( W^{n,\infty}(\Omega) \). The second inequality follows from {\color{black}Condition (ii) in Assumption \ref{asspde}}.

For the traditional DeepONet \citep{lu2021learning}, \( \fD \) is a sampling method, i.e., \[ \fD(f) = (f(\vx_1), f(\vx_2), \cdots, f(\vx_m)) ,\] and then a proper basis is found such that \( \fP \circ \fD(f) = \sum_{j=1}^m \alpha_j(\fD(f)) \phi_j(\vx) \). To achieve this reduction in \( W^{n,\infty}(\Omega) \), one approach is to apply pseudo-spectral projection, as demonstrated in \cite{maday1982spectral, canuto1982approximation} and used in the approximation analysis of Fourier Neural Operators (FNO) in \cite{kovachki2021universal}. {\color{black} Other reductions from infinite- to finite-dimensional function spaces are possible, e.g., general basis expansions or finite element discretizations \citep{hao2024newton}.  
However, for a generic basis, the coefficients are not available in closed form; they must be computed numerically before training, which adds significant overhead.  
Finite element schemes, typically require derivative information on \(f\) to achieve Sobolev measure precision.  
In contrast, the pseudo-spectral projection provides high‑order accuracy in \(W^{n,\infty}(\Omega)\) without extra preprocessing, making it the most practical choice for the present analysis.}

\begin{proposition}\label{functional}
    For any \( f \in H^{s}(\sT^d) \) with \( s  > n + \frac{d}{2} \), where \( H^{s}(\sT^d) \) denotes the functions in \( H^{s}(\Omega) \) with periodic boundary conditions, define
    \[
    \fD(f) = (f(\vx_{\vnu}))_{\vnu \in \{0,1,\ldots,2N\}^d}
    \]
    for \( \vx_{\vnu} = \frac{\vnu}{2N+1} \). Thus, there exists a Lipschitz continuous map \( \fP : [-M, M]^m \to H^{s'}(\sT^d) \) such that
\begin{equation}
    \|f - \fP \circ \fD(f)\|_{W^{n,\infty}(\Omega)} \leq C N^{s' - s} \|f\|_{H^{s}(\Omega)},\notag
\end{equation}
for any \( s' \in \left(n + \frac{d}{2}, s\right) \), where \( M = \|f\|_{L^\infty(\Omega)} \), \( C \) is independent of \( N \) and \( f \), and \( m = (2N+1)^d \). Furthermore, the Lipschitz constant of \( \fP \) is \( \vO(N^{s'}) \).
\end{proposition}

\begin{remark}
   In \cite{lanthaler2022error, kovachki2021universal}, it is shown that the decay rate of the decoder can be exponential if the coefficients of the expansion of the input function decay exponentially, as in the case of holomorphic functions, which contain a very low-dimensional structure \citep{opschoor2022exponential}.

   {\color{black} In Proposition~\ref{functional} we assume periodic boundary conditions because the pseudo-spectral projection requires periodicity.  
When the source term \(f\) is not periodic, one may apply an extension operator that maps \(f\) to a periodic function \(\bar f\) defined on an enlarged cuboid.  
This strategy has two drawbacks: (i) it demands sampling outside the original domain and (ii) a suitable extension may not exist for general geometries.  
A detailed discussion appears in \citet{kovachki2021universal}.  
Removing the periodicity assumption entirely therefore remains an open problem and is left for future work.  
Throughout this paper we also adopt a uniform grid; more flexible discretizations—such as multigrid sampling—work well in practice \citep{he2023mgno}, but their theoretical analysis is likewise deferred to future work.
}
\end{remark}

Based on Proposition \ref{functional}, we can reduce functional learning to function learning. Note that, even though we are considering physics-informed training for operator learning, we do not need derivative information in the encoding process. Specifically, the projection \( \fD \) defined as \( \fD(f) = (f(\vx_{\vnu}))_{\vnu \in \{0,1,\ldots,2N\}^d} \) is sufficient for tasks like solving PDEs. Therefore, when designing operator neural networks for solving PDEs, it is not necessary to include derivative information of the input space in the domain. However, whether adding derivative information in the encoding helps improve the effectiveness of physics-informed training remains an open question, which we will address in future work.

\subsection{Function learning in Sobolev spaces}
In this part, we will use a neural network to approximate each functional \( c_k(\fG_* \circ \fP \circ \fD(f)) \), i.e., we aim to make 
\begin{equation}
    \sup_{f \in \fX} \left| c_k(\fG_* \circ \fP \circ \fD(f)) - \fB(\fD(f); \vtheta_{1,k}) \right|\notag
\end{equation}
as small as possible. For simplicity, we consider all functions in \( \fX \) with \( \|f\|_{L^\infty(\Omega)} \leq M \) for some \( M > 0 \). We then need to construct a neural network in \( [-M, M]^m \) such that
\begin{equation}
    \sup_{\vz \in [-M, M]^m} \left| c_k(\fG_* \circ \fP(\vz)) - \fB(\vz; \vtheta_{1,k}) \right|\notag
\end{equation}
is small. The idea for achieving this approximation is to first analyze the regularity of \( c_k(\fG_* \circ \fP(\vz)) \) with respect to \( \vz \). Once the regularity is established, we can construct a neural network that approximates it with a rate that depends on the regularity. {\color{black} Specifically, we prove that
\(c_k\bigl(\mathcal{G}_* \circ \mathcal{P}(\mathbf{z})\bigr)\)
is Lipschitz‑continuous; this follows from condition (ii) in
Assumption~\ref{asspde}.
If the operator enjoys higher smoothness, the approximation rate can be
further improved.
Such enhanced regularity is indeed available for certain PDEs; see, for
instance, the results of \citet{cohen2010convergence,cohen2011analytic}.}

\begin{lemma}\label{Lip}
    Suppose \( \fG_* \) is an operator satisfying Assumption \ref{asspde} with domain {\color{black}\( \fX \subset W^{n,\infty}(\sT^d)\subset H^{s'}(\sT^d) \) (i.e. \( s' > n + \frac{d}{2} \)),} and \(\sup_{f \in \fX} \|f\|_{L^\infty(\Omega)} \leq M \) for some \( M > 0 \). Let \( c_k \) be the functional defined in Proposition \ref{operator}, and \( \fP \) be the map defined in Proposition \ref{functional}. Then, \( c_k(\fG_* \circ \fP(\vz)) \in W^{1,\infty}([-M,M]^m) \), with its norm being \( \vO(m^{\frac{s'}{d}}) \).
\end{lemma}

\begin{proof}
    For any \( \vz_1, \vz_2 \in [-M, M]^m \), we have
    \begin{align}
        &|c_k(\fG_* \circ \fP(\vz_1)) - c_k(\fG_* \circ \fP(\vz_2))|
        \notag\\\leq& C \|\fG_* \circ \fP(\vz_1) - \fG_* \circ \fP(\vz_2)\|_{W^{n,\infty}(\Omega)},\notag
    \end{align}
    which holds because \( c_k \) is a bounded linear functional in \( W^{n-1,\infty}(\Omega) \), and thus also a bounded linear functional in \( W^{n,\infty}(\Omega) \).

    Next, based on the Lipschitz condition of \( \fG_* \) in Condition (ii) of Assumption \ref{asspde}, we have
    \begin{align}
        &C \|\fG_* \circ \fP(\vz_1) - \fG_* \circ \fP(\vz_2)\|_{W^{n,\infty}(\Omega)}
        \notag\\\leq &C \|\fP(\vz_1) - \fP(\vz_2)\|_{W^{n,\infty}(\Omega)}
        \leq C \|\fP(\vz_1) - \fP(\vz_2)\|_{H^{s'}(\Omega)},\notag
    \end{align}
    where the last inequality is due to the Sobolev embedding theorem and \( s' > n + \frac{d}{2} \).

    Finally, based on the Lipschitz condition of \( \fP \) proved in Proposition \ref{functional}, we have
    \begin{align}
        C \|\fP(\vz_1) - \fP(\vz_2)\|_{H^{s'}(\Omega)} \leq C m^{\frac{s'}{d}} |\vz_1 - \vz_2|.\notag
    \end{align}

    Therefore, \( c_k(\fG_* \circ \fP(\vz)) \) is Lipschitz continuous in \( [-M, M]^m \). Since \( [-M, M]^m \) is a convex set, by \cite{heinonen2005lectures}, we conclude that \( c_k(\fG_* \circ \fP(\vz)) \in W^{1,\infty}([-M,M]^m) \), with its norm being \( \vO(m^{\frac{s'}{d}}) \).
\end{proof}

Now, we can establish a neural network to approximate \( c_k(\fG_* \circ \fP(\vz)) \). For shallow neural networks, the result from \cite{mhaskar1996neural} is as follows:

\begin{lemma}[\cite{mhaskar1996neural}]\label{shallow}
    Suppose \( \sigma \) is a continuous non-polynomial function, and \( \sK \) is a compact subset of \( \mathbb{R}^m \). Then, there exist positive integers \( W \), and constants \( w_k, \zeta_k, c_k \) for \( k = 1, \ldots, W \), such that for any \( v \in W^{1,\infty}(\sK)\),
    \begin{equation}
        \left\| v - \sum_{k=1}^{W} c_k \sigma\left( \vw_k \cdot \vz + \zeta_k \right) \right\|_{L^\infty(\sK)} \leq C W^{-1 / m} \|v\|_{W^{1,\infty}(\sK)}.\notag
    \end{equation}
\end{lemma}

For deep neural networks, the result from \cite{shen2022optimal} is as follows:

\begin{lemma}[\cite{shen2022optimal}]\label{deep}
    Given a continuous function \( f \in W^{1,\infty}(\sK) \), where \( \sK \) is a compact subset of \( \mathbb{R}^m \), for any \( W>m W^{\frac{1}{m}}, L \in \mathbb{N}^{+} \), there exists a \( \sigma_1 \)-NN with width \( C_1 3^m W \) and depth \( C_2 L \) such that
    \begin{align}
    \|f(\vz) - \phi(\vz)\|_{L^\infty(\sK)} \le C_3 \left( \left( W^2 L^2 \log^3 (W+2) \right)^{-1 / m} \right) \|f\|_{W^{1,\infty}(\sK)},\notag
    \end{align}
    where \( C_1, C_2 \), and \( C_3 \) are constants independent of \( m, W \) and \( L \).
\end{lemma}

\begin{remark}
For shallow neural networks, based on Lemma \ref{shallow}, we know that using \( (q) \) parameters can only achieve an approximation rate of \( \vO(q^{-\frac{1}{m}}) \) for each \( c_k(\fG_* \circ \fP(\vz)) \). However, for deep neural networks, the approximation rate can be \( \vO(q^{-\frac{\lambda}{m}}) \), where \( \lambda \in [1, 2] \). Here, \( \lambda = 1 \) corresponds to the shallow or very wide case (i.e., the depth is logarithmic in the width), and \( \lambda = 2 \) corresponds to the very deep case (i.e., the width is logarithmic in the depth). This is because, in Lemma \ref{deep}, a \( \sigma_1 \)-NN with width \( C_1 W \) and depth \( C_2 L \) has a number of parameters \(\fP= \vO(W^2 L) \), and the approximation rate is \( \vO\left( \left( W^2 L^2 \log^3(W+2) \right)^{-1/m} \right) \). Thus, we observe that the approximation rate satisfies:
\[
(W^2L)^{-2/m} = \fP^{-\lambda/m}
\]due to  \[(W^2L)^{-2/m}\le(W^2L)^{-2/m}\le (W^2L)^{-1/m}.\]
When $L \gg W$, $\lambda$ approaches 2, and when $L \ll W$, $\lambda$ approaches 1. More specifically, if we denote $L = W^\alpha$, then $\fP = W^{\alpha+2}$, and the approximation rate is $W^{-(2+2\alpha)/m} = \fP^{-\lambda/m}$, where $\lambda = \frac{2+2\alpha}{\alpha+2}$.
    
\end{remark}

In \cite{yarotsky2020phase, yang2024near}, it is mentioned that an approximation rate better than \( \vO(q^{-\frac{1}{m}}) \) may cause the neural network approximator to become discontinuous. However, unlike the approximation of the trunk network, we do not require the approximation of the branch network to be continuous. Thus, we can establish a deep neural network for the approximation of the branch network and obtain the benefits of a deep structure, as indicated by the improved approximation rates.

\begin{proposition}\label{function}
    Suppose \( \fG_* \) is an operator satisfying Assumption \ref{asspde} with domain \( \fX \subset H^{s'}(\sT^d) \), where \( s' > n + \frac{d}{2} \), and \(\sup_{f \in \fX} \|f\|_{L^\infty(\Omega)} \leq M \) for some \( M > 0 \). For any \( k \in \{1, 2, \ldots, p\} \), let \( c_k \) be the functional defined in Proposition \ref{operator}, and \( \fP \) be the map defined in Proposition \ref{functional}. Then for any \( \lambda \in [1, 2] \) and sufficiently large \( q \), there exists $\sigma_1$-NNs \( \fB(\vz; \vtheta_k) \) with \( C_19^m q \) parameters such that
\begin{equation}
    \left\|c_k(\fG_* \circ \fP(\vz)) - \fB(\vz; \vtheta_k)\right\|_{L^\infty([-M,M]^m)} \leq C_2 m^{\frac{s'}{d}} q^{-\frac{\lambda}{m}},\notag
\end{equation}
where \( C_1, C_2 \) are independent of both \( m \) and \( q \). Furthermore, for \( \lambda = 1 \), shallow neural networks can achieve this approximation rate. As \( \lambda \) approaches 2, the ratio between the width and depth of \( \fB \) decreases, meaning the network structure becomes deeper.
\end{proposition}

\begin{proof}
    For \( \lambda = 1 \), the proof of Proposition \ref{function} can be obtained by combining Lemmas \ref{Lip} and \ref{shallow}. For \( \lambda > 1 \), the proof of Proposition \ref{function} can be obtained by combining Lemmas \ref{Lip} and \ref{deep}.
\end{proof}

The result in Proposition \ref{function} exhibits a curse of dimensionality when \( m \) is large. One potential approach to mitigate this issue is to incorporate a proper measure \( \mu_{\fX} \) instead of using \( \sup_{f \in \fX} \) in the error analysis. However, since this paper focuses on establishing the framework for generalization error in physics-informed training for operator learning and investigating whether deep structures offer benefits, addressing the curse of dimensionality in physics-informed training is left for future work. Solving this challenge would likely require \( \fX \) and \( \fY \) to exhibit smoother structures.

{\color{black}\section{Experiments and Further Discussion on Approximation Error}

Recall that 
\begin{align}
    \fG(s; \vtheta) := \sum_{k=1}^p 
    \underbrace{\fB(\fD(s); \vtheta_{1,k})}_{\text{branch}} \;
    \underbrace{\fT(\vy; \vtheta_{2,k})}_{\text{trunk}}.
    \notag
\end{align}

In Theorem \ref{approximation}, we showed that the complexity of the trunk network does not appear explicitly in the approximation error. The reason is that the trunk network essentially approximates the basis functions in the $\fY$ space. While a more complex trunk network can approximate each basis function more accurately, the overall approximation ability is governed by the number of trunk networks $p$. Even if each trunk network is chosen to be highly complex, the approximation rate of the trunk part remains 
\[
    p^{-\tfrac{n-2}{d}}
\]
based on \cite{devore1989optimal}. Thus, increasing the complexity of the trunk network has limited effect for fixed $p$, but it significantly increases the training complexity.

In contrast, the branch network approximates each functional, and this contributes the dominant error in Theorem \ref{approximation}, since it involves approximation in a infinite-dimensional space. After balancing the terms, the approximation rate is of order
\[
    \mathcal{O}\!\left[p^{2/d}\left(\frac{\ln\!\ln q}{\lambda \ln q}\right)^{\frac{s-s'}{d}}\right].
\]
Therefore, a more expressive structure is required in the branch network to reduce the error.

There are cases where increasing complexity can actually worsen training performance. For instance, when the input operator has a low-dimensional structure, adding complexity may slightly reduce the approximation error, but the dominant contribution still comes from the trunk network. In such situations, making the branch network more complex only increases the training cost without providing significant benefits, and may even harm performance.

In summary, in most cases the branch network requires higher complexity, while each trunk network can remain relatively simple. The reason is that the trunk network mainly serves to represent basis functions, whereas the branch network is responsible for learning the infinite-dimensional mapping. For this reason, more data and training effort should be devoted to the branch network, which constitutes part of our future work.

We illustrate the operator learning framework on the one-dimensional Poisson equation:
\begin{align}
& \Delta u(x) = f(x), \quad x \in [0,1], \nonumber \\
& u(0) = u(1) = 0, \label{eq:1D_Poisson}
\end{align}
with Dirichlet boundary conditions. The task is to learn the operator mapping the source term \(f\) to the steady-state solution \(u\):
\[
\fG_*: L^2((0,1);\mathbb{R}) \;\longrightarrow\; H_0^2((0,1);\mathbb{R}), 
\qquad f \mapsto u.
\]

We generate the source term \(f(x)\) from polynomials of degree at most 3. Both the input functions \(f(x)\) and the corresponding solutions \(u(x)\) are discretized on \(N=100\) equispaced points in \([0,1]\). A total of \(M=1000\) samples are used for training and an additional 200 for testing. The DeepONet is employed, consisting of a branch network that processes the input function \(f\) and a trunk network that encodes the spatial coordinate \(x\). The loss function is
\begin{equation}
\begin{aligned}
\min_{\vtheta \in \mathbb{R}^p}\; L(\vtheta) 
&= \frac{1}{2M} \sum_{i=1}^M \left[\frac{1}{N}\sum_{j=1}^N 
\Big(\Delta \fG(f_i;\vtheta)(x_j) - f_i(x_j)\Big)^2  + \frac{1}{2}\sum_{b=1}^2 \Big(\Delta \fG(f_i;\vtheta)(x_b)\Big)^2\right],
\end{aligned}
\label{pinn_loss}
\end{equation}
where the second term enforces the Dirichlet boundary conditions at \(x_b=0,1\).

The experimental results indicate that allocating more parameters to the branch network yields significantly better accuracy than enlarging the trunk network. Figure \ref{fig:Poisson_result_1} shows that, when the branch capacity is fixed, enlarging the trunk network up to hundreds of thousands of parameters does not bring measurable accuracy gains. In contrast, Figure \ref{fig:Poisson_result_2} demonstrates that enlarging the branch network while keeping the trunk small markedly improves accuracy, even under relatively modest overall parameter budgets. This highlights that branch-network capacity dominates trunk-network capacity in determining DeepONet performance for this Poisson operator learning task.

\begin{figure}[t]
    \centering
    \includegraphics[scale=0.45]{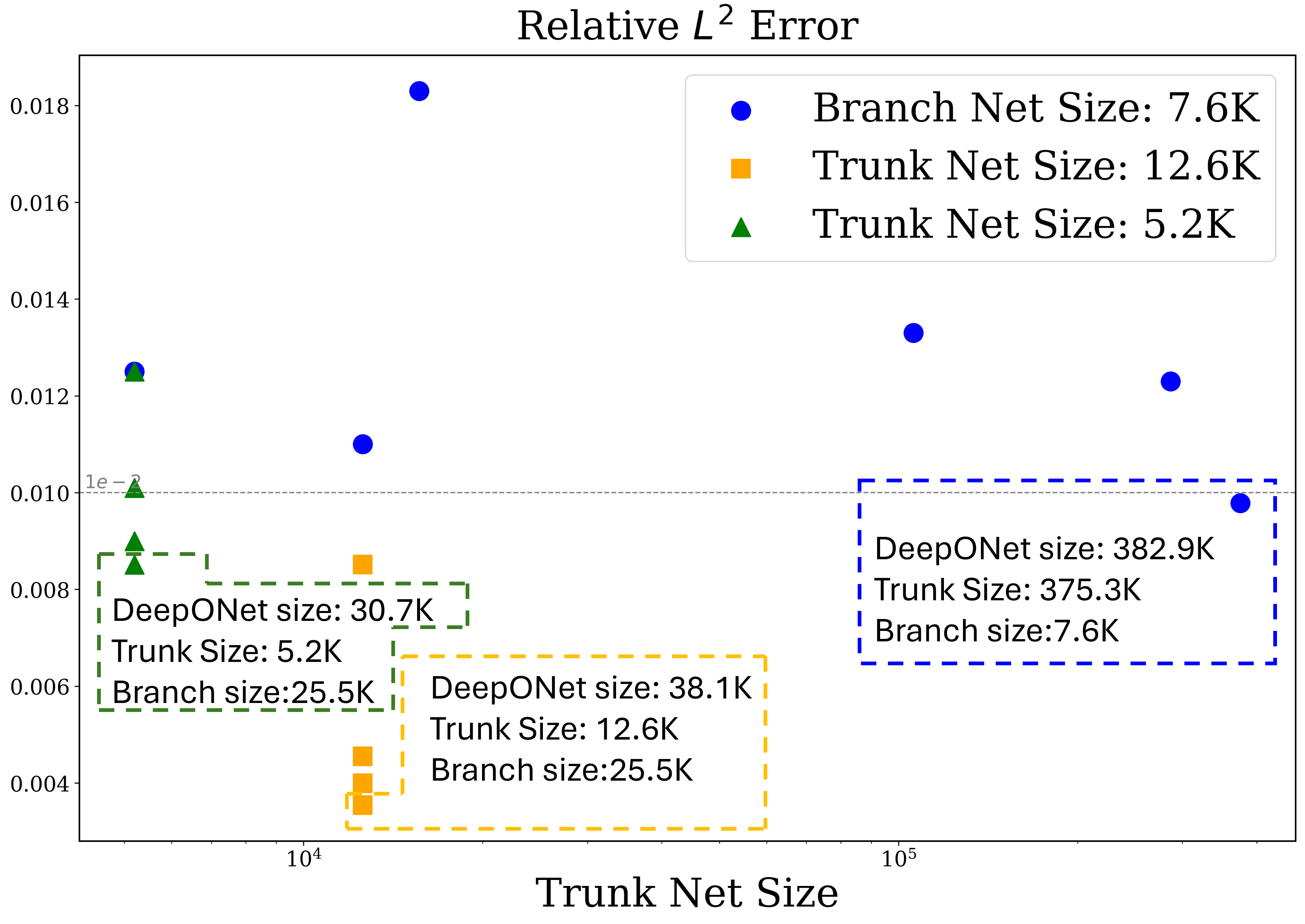}
    \caption{Effect of parameter allocation in DeepONet. Each marker denotes a trained model; the vertical axis shows test error (log scale). Blue circles: small branch (7.6K params) with a large trunk (up to 375.3K; total $\approx$382.9K). Orange squares: moderate trunk (12.6K) with a larger branch (25.5K; total $\approx$38.1K). Green triangles: small trunk (5.2K) with a larger branch (25.5K; total $\approx$30.7K). Enlarging the branch while keeping the trunk small achieves markedly lower error with a far smaller overall parameter budget, whereas increasing the trunk alone yields limited accuracy gains.}
    \label{fig:Poisson_result_1}
\end{figure}

\begin{figure}[t]
    \centering
    \includegraphics[scale=0.45]{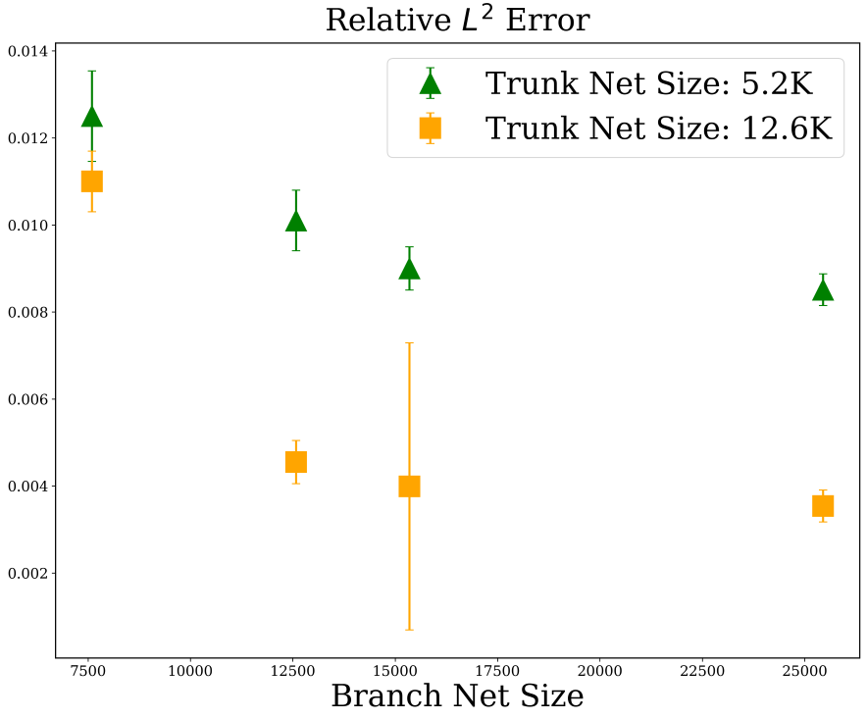}
    \caption{Test error (vertical axis) for DeepONet under two trunk capacities while varying the branch network. Green triangles: trunk = 5.2K parameters. Orange squares: trunk = 12.6K parameters. Markers denote the mean over multiple random initializations; vertical bars indicate $\pm 1$ standard deviation. Across settings, enlarging the branch yields the dominant accuracy gains, whereas increasing the trunk from 5.2K to 12.6K offers only limited additional improvement.}
    \label{fig:Poisson_result_2}
\end{figure}

}
\section{Generalization error}
Recall that the generalization error is defined as:
\begin{align}
    \mathbb{E} \left[ L_D(\vtheta_S) - L_S(\vtheta_S) \right],\notag
\end{align}
where \( \vtheta_D = \arg\min L_D(\vtheta) \) and \( \vtheta_S = \arg\min L_S(\vtheta) \). In the definition of \( L_S(\vtheta) \), there are two types of samples: one is for the sample points \( \vy \), and the other is for the input function \( f \). To bound the generalization error, we introduce \( L_M(\vtheta) \), defined as:
\begin{align}
    L_M(\vtheta) := \frac{1}{M} \sum_{i=1}^M \int_{\Omega} |\fL \fG(f_i;\vtheta)(\vy) + f_i(\vy)|^2 \, \mathrm{d} \vy.\notag
\end{align}

Thus, we can write:
\begin{align}
    &\mathbb{E} \left[ L_D(\vtheta_S) - L_S(\vtheta_S) \right] \notag\\\leq& \mathbb{E} \left[ L_D(\vtheta_S) - L_M(\vtheta_S) \right] + \mathbb{E} \left[ L_M(\vtheta_S) - L_S(\vtheta_S) \right].\notag
\end{align}
Next, we will divide the generalization error into two separate components and analyze them independently, rather than considering them together. This approach offers two key benefits. First, it allows us to clearly distinguish the two sources of error in the final results: one arising from sampling in the domain of output spaces, and the other from sampling in the input spaces. Second, it simplifies the analysis because the pairs $(f_i, \vy_j)$ are not independent and identically distributed (i.i.d.) across the entire space $\mathcal{X} \times \Omega$. Specifically, each $f_i$ appears $P$ times in the total sampling, which introduces dependencies that would complicate the analysis if both components were considered together.

Recall that \begin{align}
    \fG(s; \vtheta) := \sum_{k=1}^p \underbrace{\fB(\fD(s); \vtheta_{1,k})}_{\text{branch}} \underbrace{\fT(\vy; \vtheta_{2,k})}_{\text{trunk}},
    \notag
\end{align}we denote that $\fT(\vy; \vtheta_{2,k})$ is the network with width $W_T$ and depth $L_T$, and $\fB(\fD(s); \vtheta_{1,k})$ is the network with width $W_B$ and depth $L_B$. In order to bound $\mathbb{E} \left[ L_M(\vtheta_S) - L_S(\vtheta_S) \right]$, we require the following assumptions for the space of $\vtheta$ and $f_i \in \fX$.

\begin{assumption}\label{first lip}
    \textbf{(i) Boundedness:} There exists a functional \(\Psi(f)\) and a function \(\Phi(\vy)\) defined on \(\fX\) and \(\Omega\), respectively, such that  
\[
|\fL \fG(f; \vtheta)(\vy) + f(\vy)| \leq \Psi(f), \quad \left(\mathbb{E}_{f \sim \mu_{\fX}}[\Psi^2(f)]\right)^{\frac{1}{2}} \leq G
\]
for all \(\vy \in \Omega\). Furthermore,  
\[
|\fL \fG(f; \vtheta)(\vy) + f(\vy)| \leq \Phi(\vy), \quad \left(\mathbb{E}_{\vy \sim \operatorname{Unif}(\Omega)}[\Phi^2(\vy)]\right)^{\frac{1}{2}} \leq G
\]
for all \(f \in \fX\). We denote the set of all parameters satisfying this assumption as \(\Theta\).

    \textbf{(ii) Polynomial:} There exists a polynomial function $F(\vz)$ with degree $d_{\fL}$ such that 
    \[
    \fL u(\vy) = F\left(u, (\partial_i u)_{i=1}^d, (\partial_{ij} u)_{i,j=1}^d\right)
    \]
    for all $\vy \in \Omega$.
\end{assumption}

\begin{remark}
The first assumption is reasonable because it represents a weaker form of the condition that 
\[
|\mathcal{L} \mathcal{G}(f; \boldsymbol{\theta})(\boldsymbol{y}) + f(\boldsymbol{y})|
\] 
is bounded, given that \( \mu_{\mathcal{X}} \) is a finite measure. This assumption is justified since 
\[
|\mathcal{L} \mathcal{G}^*(f)(\boldsymbol{y}) + f(\boldsymbol{y})| = 0,
\]
and the neural networks we aim to learn are expected to approximate \( \mathcal{G}^*(f)(\boldsymbol{y}) \) closely. Consequently, for the neural network \( \mathcal{G}(f; \boldsymbol{\theta})(\boldsymbol{y}) \), it should hold that 
\[
|\mathcal{L} \mathcal{G}(f; \boldsymbol{\theta})(\boldsymbol{y}) + f(\boldsymbol{y})|
\] 
remains small. This assumption is significantly weaker than the one used in the analysis of \cite{lanthaler2022error}, where they require \( \mathcal{G}(f; \boldsymbol{\theta})(\boldsymbol{y}) \) to be Lipschitz with respect to \( \boldsymbol{\theta} \). However, the Lipschitz constant can be extremely large for deep neural networks, making their assumption much more restrictive. The second assumption is satisfied for many nonlinear PDEs, such as the operator in Allen-Cahn equation with a source term\[\fL u:=\frac{\partial u}{\partial t}-(\varepsilon^2 \Delta u + u^3 - u) \] and that in the Burgers' equation 
\[
\fL u:=\frac{\partial u}{\partial t} + u \cdot \nabla u - \nu \Delta u.
\]
\end{remark}

The approach to bounding the generalization error begins with using the Rademacher complexity. To achieve this, we extend the concept of generalization error to functional spaces. Building on the proofs in \cite{shalev2014understanding} and \cite{jiao2021error}, the lemma utilized in this paper, which relates to the Rademacher complexity, does not depend on the finite-dimensionality assumption of the input space.
\begin{definition}[{Rademacher complexity \cite{anthony1999neural}}]\label{defrad}
				Given a set of samples  $S=\{\vz_1,\vz_2,$ $\ldots,\vz_m\}$ on a domain $\fZ$, and a class $\fF$ of real-valued functions or functionals defined on $\fZ$, the empirical Rademacher complexity of $\fF$ in $S$ is defined as \[\mathbf{R}_S(\fF):=\frac{1}{m}\mathbb{E}_{\Xi_m}\left[\sup_{f\in\fF}\sum_{i=1}^m\xi_if(\vz_i)\right],\]where $\Xi_m:=\{\xi_1,\xi_2,\ldots,\xi_m\}$ is a set of $m$ independent random samples drawn from the Rademacher distribution, i.e., $\mathbb{P}(\xi_i=+1)=\mathbb{P}(\xi_i=-1)={1}/{2}$, for $i=1,2,\ldots,m.$ 
			\end{definition}

The following lemma bounds the expected value of the largest gap between the expected value of a function or functional $f$ and its empirical mean by twice the expected Rademacher complexity.
			\begin{lemma}[{\cite[Lemma 26.2]{shalev2014understanding}}]\label{connect1}
				Let $\fF$ be a set of functions or functionals defined on $\fZ$. Then \[\mathbb{E}_{S\sim \rho^m}\sup_{f\in\fF}\left(\frac{1}{m}\sum_{i=1}^mf(\vz_i)-\mathbb{E}_{\vz\sim\rho} f(\vz)\right)\le 2\mathbb{E}_{S\sim\rho^m} \mathbf{R}_S(\fF),\]where $S = \{\vz_1, \vz_2, \ldots, \vz_m\}$ is a set of $m$ independent random samples drawn from the distribution $\rho$.
			\end{lemma}
  
		Lemma below shows that composing functions from a function or functional class $\mathcal{F}$ with a Lipschitz function $w$ does
not blow up the Rademacher complexity.	
\begin{lemma}{\cite[Lemma 26.9]{shalev2014understanding}}\label{jiao}

For any $m\ge 0$, assume that $w_i: \sR \rightarrow \sR$ and the Lipschitz constant of $w_i$ is $B_i$, then for any function or functional
class $\fF$, it holds that
\begin{equation}
\frac{1}{m}\mathbb{E}_{\Xi_m}\left[\sup_{f\in\fF}\sum_{i=1}^m\xi_iw_i\circ f(\vz_i)\right]\le \frac{1}{m}\mathbb{E}_{\Xi_m}\left[\sup_{f\in\fF}\sum_{i=1}^m\xi_i B_i f(\vz_i)\right],\notag
\end{equation}
where $S = \{\vz_1, \vz_2, \ldots, \vz_m\}$ is a set of independent random samples drawn from the distribution $\mathcal{D}$.
			\end{lemma}

Based on these two lemmas, we can bound the generalization error using the Rademacher complexity of specific classes.

\begin{proposition}\label{first lemma}
    Suppose that {\color{black}Condition (i) in Assumption \ref{first lip}} holds and denote 
    \begin{align}
    \fC_{f_*}&:=\{h_2(\vy)= \fL \fG(f_*;\vtheta)(\vy)\mid \vtheta\in\Theta,\vy\in\Omega\},\notag\\\fJ_{\vy_*}&:=\{h_3(\vy_*)= \fL \fG(f;\vtheta)(\vy_*)\mid \vtheta\in\Theta,f\in\fX\}\notag
    \end{align}for $f_*\in\fX,\vy_*\in\Omega$.
    Then, we have
    \begin{align}
        \mathbb{E} \sup_{\vtheta\in\Theta}\left[ L_M(\vtheta) - L_S(\vtheta) \right] &\leq 4\mathbb{E}_{f \sim \mu_{\fX}} \left[  \Psi(f) \mathbb{E}_{\fY \sim \operatorname{Unif}(\Omega)^P} \mathbf{R}_{\fY}(\fC_{f}) \right],\notag\\\mathbb{E} \sup_{\vtheta\in\Theta}\left[ L_D(\vtheta) - L_M(\vtheta) \right] &\le 4\mathbb{E}_{\vy\sim \operatorname{Unif}(\Omega)} \left[\Phi(\vy)\mathbb{E}_{\fX^M\sim \mu_{\fX}^M}\mathbf{R}_{\fX^M} (\fJ_{\vy})\right]\notag
    \end{align}
    where $\mu_{\fX}$ is the distribution measure of the input space $\fX$.
\end{proposition}

\begin{proof}
    There are two inequalities in the proof, and their derivations are similar. Without loss of generality, we present the proof for the first inequality here. Set for any $f_i\in\fX$ and \begin{align}\fB_{f_i}&:=\{h_1(\vy)= (\fL \fG(f_i;\vtheta)(\vy) + f_i(\vy))^2\mid \vtheta\in\Theta\}\notag\\\fC_{f_i}&:=\{h_2(\vy)= \fL \fG(f_i;\vtheta)(\vy) + f_i(\vy)\mid \vtheta\in\Theta\}\notag\end{align}then we have \begin{align}
        &\mathbb{E}_{\fY\sim \operatorname{Unif}(\Omega)^P}\sup_{\vtheta\in\Theta}\left[\int_{\Omega} |\fL \fG(f_i;\vtheta)(\vy) + f_i(\vy)|^2 \, \mathrm{d} \vy-\frac{1}{P}\sum_{j=1}^P |\fL \fG(f_i;\vtheta)(\vy_j)+ f_i(\vy_j)|^2\right]\notag\\\le & 2\mathbb{E}_{\fY\sim \operatorname{Unif}(\Omega)^P}\mathbf{R}_{\fY}(\fB_{f_i})\le 4\Psi(f_i)\mathbb{E}_{\fY\sim \operatorname{Unif}(\Omega)^P}\mathbf{R}_{\fY}(\fC_{f_i}).\notag
    \end{align}

    The first inequality follows directly from Lemma \ref{connect1}. The second inequality is due to the fact that $x^2$ is Lipschitz continuous when $|x| \leq \frac{1}{2} L_*$, with a Lipschitz constant equal to $L_*$. By Assumption \ref{first lip}, the Lipschitz constant of $\phi$ is $2\Psi(f_i)$ where
\[
\fB_{f_i} := \{ \phi(h_2) \mid h_2 \in \fC_{f_i} \}.
\]
Then based on Lemma \ref{jiao}, we derive the second inequality. 

Finally, we have that 
\begin{align}
    & \mathbb{E}\sup_{\vtheta\in\Theta} \left[ L_M(\vtheta) - L_S(\vtheta) \right] \notag \\
    =& \mathbb{E}_{f \sim \mu_{\fX}} \mathbb{E}_{\fY \sim \operatorname{Unif}(\Omega)^P} \sup_{\vtheta\in\Theta}\left[ \int_{\Omega} \left| \fL \fG(f; \vtheta)(\vy) + f(\vy) \right|^2 \, \mathrm{d} \vy - \frac{1}{P} \sum_{j=1}^P \left| \fL \fG(f; \vtheta)(\vy_j) + f(\vy_j) \right|^2 \right] \notag \\
    \leq & 4  \mathbb{E}_{f \sim \mu_{\fX}}\left[\Psi(f) \mathbb{E}_{\fY \sim \operatorname{Unif}(\Omega)^P} \mathbf{R}_{\fY}(\fC_{f})\right],\notag
\end{align}
and we conclude the proof.
\end{proof}

{\color{black}To bound the Rademacher complexity, we proceed via the pseudo‑dimension~\cite{pollard1990empirical}.  Dudley’s entropy integral~\cite{anthony1999neural} expresses the Rademacher complexity in terms of the $\varepsilon$‑covering number of the hypothesis class; Lemma~\ref{cover dim} then bounds this covering number by the pseudo‑dimension.  A direct computation of the pseudo‑dimension yields an explicit bound, stated below, while all intermediate steps are worked out in the appendix.
}

\begin{definition}[pseudo-dimension \cite{pollard1990empirical}]\label{Pse}
		Let $\fF$ be a class of functions from $\fH$ to $\sR$. The pseudo-dimension of $\fF$, denoted by $\text{Pdim}(\fF)$, is the largest integer $m$ for$\footnote{For notational simplicity we reuse the symbol $m$ in this section and in the proof of Proposition~\ref{boundvc}. Here $m$ denotes the sample size, whereas in the approximation‑error analysis (e.g., Lemma~\ref{Lip}) it refers to $\dim(\fD)$.}$ which there exists $(x_1,x_2,\ldots,x_m,y_1,y_2,\ldots,y_m)\in\fH^m\times \sR^m$ such that for any $(b_1,\ldots,b_m)\in\{0,1\}^m$ there is $g\in\fF$ such that $\forall i: g\left(x_i\right)>y_i \Longleftrightarrow b_i=1.$
	\end{definition}

     \begin{proposition}\label{boundvc}
           Suppose that {\color{black}Condition (ii) in Assumption~\ref{first lip}} holds. Then
\begin{align}
    \max\{\text{Pdim}(\fC_f),\text{Pdim}(\fJ_{\vy})\} \leq C p \big( L_T^3 W_T^2 + L_B^2 W_B^2 \big) \log_2(L_T L_B W_B W_T)\notag
\end{align} for  all $f\in\fX$ and $\vy\in\Omega$, 
where \( C \) is a constant that depends logarithmically on \( d_\fL \), as defined in {\color{black} Condition (ii) in Assumption~\ref{first lip}}.
      \end{proposition}

Now we can prove our main theorem in this section:
\begin{theorem}\label{genn}
    Suppose that Assumption \ref{first lip} holds. Then, we have
    \begin{align}
        &\mathbb{E} \left[ L_D(\boldsymbol{\theta}_S) - L_S(\boldsymbol{\theta}_S) \right]\notag\\ \leq& C G^2 \left( p \left( L_T^3 W_T^2 + L_B^2 W_B^2 \right) \log_2(L_T L_B W_B W_T) \right)^{\frac{1}{2}} \left( \frac{\log P}{\sqrt{P}} + \frac{\log M}{\sqrt{M}} \right),\notag
    \end{align}
    where $G$ is defined in {\color{black} Condition (i) of Assumption \ref{first lip}}, and $C$ is independent of $M, P, W_B,$ $ W_T, L_B, L_T$, but depends on $d_{\mathcal{L}}$ as specified in Assumption \ref{first lip}.
\end{theorem}

{\color{black}Although the present bound is informative, it is not optimal with respect to the choice of sampling points.  Achieving the nearly optimal generalization error would require controlling higher‑order terms in the generalization gap, as in the analysis of \citet{gyorfi2002distribution}.  We leave this refinement to future work.}

\section{Conclusion}
In this paper, we estimated the generalization error of DeepONet in physics-informed training, demonstrating that the deep structure of DeepONet offers significant benefits at the approximation level, particularly in the branch network. Additionally, we showed that the classical DeepONet \cite{lu2021deep} retains strong approximation capabilities in physics-informed training. This means it can approximate not only the solution itself but also the derivatives of the solution when applied to solving PDEs. Furthermore, we derived bounds on the pseudo-dimension of DeepONet, including its derivatives.

This work addresses a critical gap by providing {\color{black}error estimates} applicable to a wide range of physics-informed machine learning models and applications. However, several open questions remain for future research. First, while our results indicate that a deep structure enhances the approximation rate of DeepONet, training deep networks can be challenging. This raises the unresolved question of how to balance the trade-off between deep and shallow neural networks. In \cite{yangdeeper}, this problem was explored for function learning, but its solution for operator learning remains unexplored. Second, our results show that the branch network requires a more complex structure than the trunk network in most cases. Therefore, future work should focus on designing sampling strategies \cite{luo2025imbalanced} and training methods \cite{chen2025automatic,chen2024quantifying,chen2025learn} that place greater emphasis on the branch network.

Additionally, this paper focuses solely on DeepONet, leaving open the question of whether similar results can be extended to other network architectures. {\color{black}For example, consider the class of {shallow neural operators} with \(p\) neurons mapping between Banach spaces \(\mathcal{X}\) and \(\mathcal{Y}\):
\[
\vO(a;\vartheta)
   \;=\;
   \sum_{i=1}^{p} \mathcal{A}_{i}\,
   \sigma\!\bigl(\mathcal{W}_{i}a + \mathcal{B}_{i}\bigr),
   \qquad a \in \mathcal{X},
\]
where
\(\mathcal{W}_{i} \in \mathcal{L}(\mathcal{X},\mathcal{Y})\),
\(\mathcal{B}_{i} \in \mathcal{Y}\),
\(\mathcal{A}_{i} \in \mathcal{L}(\mathcal{Y},\mathcal{Y})\),
and
\(\vartheta = \{\mathcal{W}_{i},\mathcal{A}_{i},\mathcal{B}_{i}\}_{i=1}^{p}\)
collects all parameters.
Here \(\mathcal{L}(\mathcal{X},\mathcal{Y})\) denotes the space of bounded linear operators from \(\mathcal{X}\) to \(\mathcal{Y}\), and
\(\sigma : \mathbb{R} \to \mathbb{R}\) is a point‑wise activation. Constructing the lifting operators \(\mathcal{A}_{i}\) in a way that reflects the regularity of the input and output spaces is non‑trivial for differential PDEs, and identifying conditions on the weight operators \(\mathcal{W}_{i}\) that yield provably {good approximation rates}—rather than merely universal approximation (as shown in \cite{he2023mgno})—remains an open problem.  Furthermore, although pseudo‑dimension techniques can still be used to bound the generalization error, it remains an open question which subset of the network parameters is most critical in determining that error.} Addressing these challenges will further deepen our understanding of operator learning and its practical applications.

\section*{Acknowledgement}
We would like to express our sincere gratitude to Dr.~Chuqi Chen for assistance with the experimental part and for valuable discussions that greatly improved the quality of this paper. We also thank the two anonymous reviewers for their constructive suggestions, which further enhanced the clarity and quality of the work.

\appendix
\section{Proof of Approximation Error}
\subsection{Proof of Proposition \ref{operator}}
\subsubsection{Preliminaries}
In the part, we collect the lemmas and definition that we need to use Proposition \ref{operator}. First of all, we will collect the lemmas and definition relative to Bramble--Hilbert Lemma.
		\begin{definition}[Sobolev semi-norm \cite{evans2022partial}]
			Let $n\in\sN_+$ and $1\le p\le \infty$. Then we define Sobolev semi-norm $|f|_{W^{n, p}(\Omega)}:=\left(\sum_{|\alpha|= n}\left\|D^{\valpha} f\right\|_{L^p(\Omega)}^p\right)^{1 / p}$, if $p<\infty$, and $|f|_{W^{n, \infty}(\Omega)}:=\max_{|\alpha| = n}\left\|D^{\valpha} f\right\|_{L^\infty(\Omega)}$. Furthermore, for $\vf\in W^{1,\infty}(\Omega,\sR^d)$, we define  $|\vf|_{W^{1,\infty}(\Omega,\sR^d)}:=\max_{i=1,\ldots,d}\{|f_i|_{W^{1,\infty}(\Omega)}\}$.
		\end{definition}
\begin{lemma}\label{average coe}
			Let $n\ge 1$ and $f\in W^{n,\infty}([0,1]^d)$, $\vx_0\in\Omega$ and $r>0$ such that for the ball $B:=B_{r,|\cdot|}(\vx_0)$ which is a compact subset of $ [0,1]^d$. The corresponding Taylor
			polynomial of order $n$ of $f$ averaged over $B$ can be read as \[Q^nf(\vx)=\sum_{|\valpha|\le n-1}c_{f,\valpha}\vx^{\valpha}.\]
			Furthermore,  \[c_{f,\valpha}=\sum_{|\valpha+\vbeta|\le n-1}\frac{1}{(\vbeta+\valpha)!}a_{\vbeta+\valpha}\int_{B}	D^{\valpha+\vbeta}{f}(\vx)\vy^{\vbeta}b_r(\vy)\,\mathrm{d} \vy\] with $a_{\vbeta+\valpha}\le \frac{(\valpha+\vbeta)!}{\valpha!\vbeta!}$, $b_r$ is a smooth function and \begin{align}\left|c_{f,\valpha}\right|\le C_2(n,d)\|{f}\|_{W^{n-1,\infty}(B)}.\notag
			\end{align} where $C_2(n,d)=\sum_{|\valpha+\vbeta|\le n-1}\frac{1}{\valpha!\vbeta!}$.
		\end{lemma}
		
		\begin{definition}
			Let $\Omega,~B\in\sR^d$. Then $\Omega$ is called stared-shaped with respect to $B$ if \[\overline{\text{conv}}\left(\{\vx\}\cup B\subset \Omega\right),~\text{for all }\vx\in\Omega.\]
		\end{definition}
		
		\begin{definition}
			Let $\Omega\in\sR^d$ be bounded, and define \[\fR:=\left\{r>0: \begin{array}{l}
				\text { there exists } \vx_0 \in \Omega \text { such that } \Omega \text { is } \\
				\text { star-shaped with respect to } B_{r,|\cdot|}\left(\vx_0\right)
			\end{array}\right\} .\]Then we define\[r_{\max }^{\star}:=\sup \fR \quad \text { and call } \quad \gamma:=\frac{\operatorname{diam}(\Omega)}{r_{\max }^{\star}}\]the chunkiness parameter of $\Omega$ if $\fR\not=\emptyset$.
		\end{definition}
		
		\begin{lemma}[\cite{brenner2008mathematical}]\label{BH}
			Let $\Omega\in\sR^d$ be open and bounded, $\vx_0\in\Omega$ and $r>0$ such that $\Omega$ is the stared-shaped with respect to $B:=B_{r,|\cdot|}\left(\vx_0\right)$, and $r\ge \frac{1}{2}r_{\max }^{\star}$. Moreover, let $n\in\sN_+$, $1\le p\le \infty$ and denote by $\gamma$ by the chunkiness parameter of $\Omega$. Then there is a constant $C(n,d,\gamma)>0$ such that for all $f\in W^{n,p}(\Omega)$\[\left|f-Q^n f\right|_{W^{k, p}(\Omega)} \le C(n,d,\gamma) h^{n-k}|f|_{W^{n, p}(\Omega)} \quad \text { for } k=0,1, \ldots, n\]where $Q^n f$ denotes the Taylor polynomial of order $n$ of $f$ averaged over $B$ and $h=\operatorname{diam}(\Omega)$.
		\end{lemma}

  Lastly, we list a few basic lemmas of $\sigma_2$ neural networks repeatedly applied in our main analysis.
			\begin{lemma}[\cite{yang2023nearly}]\label{sigma2}
				The following basic lemmas of $\sigma_2$ neural networks hold:
				
				(i) $f(x, y)=x y=\frac{(x+y)^2-(x-y)^2}{4}$ can be realized exactly by a $\sigma_2$ neural network with one hidden layer and four neurons.
				
				(ii) Assume $\vx^{\valpha}=x_1^{\alpha_1} x_2^{\alpha_2} \cdots x_d^{\alpha_d}$ for $\valpha \in \sN^d$. For any $N, L \in \sN^{+}$ such that $N L+2^{\left\lfloor\log _2 N\right\rfloor} \geq$ $|\valpha|$, there exists a $\sigma_2$ neural network $\phi(\vx)$ with the width $4 N+2 d$ and depth $L+\left\lceil\log _2 N\right\rceil$ such that \[\phi(\vx)=\vx^{\valpha}\] for any $\vx \in \sR^d$.
			\end{lemma}	

               \subsubsection{Partition of Unity}
We are going to divide the domain \( [0,1]^d \) into several parts and approximate \( v \) locally. First, we define the following:

\begin{definition}
\label{omega}
Given \( K, d \in \mathbb{N}^+ \), and for any \( \vm = (m_1, m_2, \ldots, m_d) \in [K+1]^d \), we define
\[
\Omega_{\vm} := \prod_{j=1}^d \Omega_{m_j},
\]
where
\[
\Omega_m := \left[\frac{m-1}{K} - \frac{1}{5K}, \frac{m}{K}\right] \cap [0,1].
\]
\end{definition}

Next, we define the partition of unity based on \( \Omega \):

\begin{definition}
Define \( s(x): \mathbb{R} \to [0,1] \) as follows:
\begin{equation}
s(x) :=
\begin{cases}
  2x^{2},                           & x \in \bigl[0, \tfrac12\bigr], \\[4pt]
  -2\bigl(x-1\bigr)^{2} + 1,        & x \in \bigl[\tfrac12, 1\bigr], \\[4pt]
  1,                                & x \in [1, 5], \\[4pt]
  -2\bigl(x-5\bigr)^{2} + 1,        & x \in \bigl[5, \tfrac{11}{2}\bigr], \\[4pt]
  2\bigl(x-6\bigr)^{2},             & x \in \bigl[\tfrac{11}{2}, 6\bigr], \\[4pt]
  0,                                & \text{otherwise}.
\end{cases}\notag
\end{equation}

\end{definition}
\begin{definition}\label{sm}
				Given $K\in\sN_+$, then we define two functions in $\sR$:\begin{align}
					s_m(x)=s\left(5Kx+6-5m\right).\notag
				\end{align}
				
				Then for any $\vm=(m_1,m_2,\ldots,m_d)\in[K+1]^d$, we define \begin{equation}
					s_{\vm}(\vx):=\prod_{j=1}^d s_{m_j}(x_j)\notag
				\end{equation} for any $\vx=(x_1,x_2,\ldots,x_d)\in\sR^d$.
			\end{definition}
			
			\begin{proposition}\label{smm}
				Given $K,d\in\sN_+$, $\{s_{\vm}(\vx)\}_{\vm\in[K+1]^d}$ defined in Definition \ref{sm} satisfies:
				
				(i): $\|s_{\vm}(\vx)\|_{L^\infty([0,1]^d)}\le 1$,~$\|s_{\vm}(\vx)\|_{W^{1,\infty}([0,1]^d)}\le 5K$ and $\|s_{\vm}(\vx)\|_{W^{2,\infty}([0,1]^d)}\le 25K^2$ for any $\vm\in[K+1]^d$.
				
				(ii): $\{s_{\vm}(\vx)\}_{\vm\in[K+1]^d}$ is a partition of the unity $[0,1]^d$ with ${\rm supp}~s_{\vm}(\vx)\cap[0,1]^d=\Omega_{\vm}$ defined in Definition \ref{omega}.
			\end{proposition}
   The proof of this can be verified directly and is omitted here.

   \subsubsection{Proof of Proposition \ref{operator}}
   \begin{proof}
       We extend \(v\) to the whole space \(\mathbb{R}^d\) using the extension operator from \cite{stein1970singular}. Without loss of generality, we continue to denote the extension by \(v\), which satisfies  
\[
   \|v\|_{H^1(\mathbb{R}^d)} \;\le\; C_* \|v\|_{H^1([0,1]^d)},
\]
where the constant \(C_*>0\) depends only on the dimension \(d\) and the domain.
 For any \( p \in \mathbb{N}_+ \), set \( K = \lceil p^{\frac{1}{d}} \rceil \), and establish the partition \( \{\Omega_{\vm}\}_{\vm \in [K+1]^d} \) as defined in Definition \ref{omega}. Based on Lemma \ref{BH}, we know that in each \( \Omega_{\vm} \), there exists a function \( q_{\vm}(\vx) \) such that
\[
\left| v - q_{\vm}(\vx) \right|_{W^{k, p}(\Omega_{\vm})} \le C K^{-n+k} |v|_{W^{n, p}(\Omega_{\vm})} \quad \text{for } k = 0, 1, 2,
\]
where
\begin{align}
    q_{\vm}(\vx) = \sum_{|\valpha| \leq n-1} c_{v, \valpha, \vm} \vx^{\valpha},\notag
\end{align}
and
\begin{align}
    c_{v, \valpha, \vm} = \sum_{|\valpha + \vbeta| \leq n-1} \frac{1}{(\vbeta + \valpha)!} a_{\vbeta + \valpha} \int_{B_{\vm}} D^{\valpha + \vbeta} v(\vx) \vy^{\vbeta} b_r(\vy) \, \mathrm{d} \vy.\notag
\end{align}
Here, \( B_{\vm} \) is a subset of \( \Omega_{\vm} \) satisfying the properties in Lemma \ref{BH}, with
\begin{align}
    \left| c_{v, \valpha, \vm} \right| \leq C_2(n,d) \| v \|_{W^{n-1,\infty}(\Omega_{\vm})}\le C_2(n,d) \| v \|_{W^{n-1,\infty}(\Omega)}.\notag
\end{align} We define the partition of unity as shown in Definition \ref{sm}, and then construct
\[
v_K(\vx) = \sum_{\vm \in [K+1]^d} q_{\vm}(\vx) s_{\vm}(\vx).
\]
We now estimate the error:
\begin{align}
    \|v - v_K\|_{H^2([0,1]^d)} &= \left\|\sum_{\vm \in [K+1]^d} v(\vx) s_{\vm}(\vx) - \sum_{\vm \in [K+1]^d} q_{\vm}(\vx) s_{\vm}(\vx)\right\|_{H^2([0,1]^d)} \notag \\
    &\leq \sum_{\vm \in [K+1]^d} \left\|v(\vx) s_{\vm}(\vx) - q_{\vm}(\vx) s_{\vm}(\vx)\right\|_{H^2([0,1]^d)} \notag \\
    &= \sum_{\vm \in [K+1]^d} \left\|v(\vx) s_{\vm}(\vx) - q_{\vm}(\vx) s_{\vm}(\vx)\right\|_{H^2(\Omega_{\vm})} \notag \\
    &\leq \sum_{\vm \in [K+1]^d} \left\|v(\vx) - q_{\vm}(\vx)\right\|_{H^2(\Omega_{\vm})} \|s_{\vm}(\vx)\|_{L^\infty(\Omega_{\vm})} \notag \\
    &\quad + \sum_{\vm \in [K+1]^d} \left\|v(\vx) - q_{\vm}(\vx)\right\|_{H^1(\Omega_{\vm})} \|s_{\vm}(\vx)\|_{W^{1,\infty}(\Omega_{\vm})} \notag \\
    &\quad + \sum_{\vm \in [K+1]^d} \left\|v(\vx) - q_{\vm}(\vx)\right\|_{L^2(\Omega_{\vm})} \|s_{\vm}(\vx)\|_{W^{2,\infty}(\Omega_{\vm})} \notag \\
    &\leq C \sum_{\vm \in [K+1]^d} K^{-n+2} \|v\|_{H^n(\Omega_{\vm})} \leq 2^dCC^* K^{-n+2} \|v\|_{H^n([0,1]^d)}.\label{combine}
\end{align}Here, in the last inequality, we used that each point in \([0,1]^d\) lies in at most \(2^d\) of the sets \(\Omega_{\mathbf{m}}\).

The remaining task is to use \( \sigma_2 \)-NNs to represent \( v_K \). Notice that
\[
v_K(\vx) = \sum_{\vm \in [K+1]^d} \sum_{|\valpha| \leq n-1} c_{v, \valpha, \vm} \vx^{\valpha} s_{\vm}(\vx).
\]

Based on Lemma \ref{sigma2}, we know that \( \vx^{\valpha} \) can be represented by a neural network with width \( 4n - 4 + 2d \) and depth \( 1 + \log_2 n - 1 \). The function \( s \) is a neural network with width 3 and one hidden layer. Therefore, \( s_{\vm} \) is a neural network with width \( 6d - 4 \) and depth \( 2 + \log_2 d - 1 \). Consequently, \( \vx^{\valpha} s_{\vm}(\vx) \) can be represented by a neural network with depth \( 3 + \log_2 d - 1 + \log_2 n - 1 \) and width \( 4n - 4 + 6d \). Furthermore, $\sum_{k=1}^p\|\vx^{\valpha} s_{\vm}(\vx)\|^2_{H^{2}(\Omega)}=\fO(p^{\frac{4}{d}})$ based on the same reason shown in \eqref{combine} and Lemma \ref{smm}.

   \end{proof}

   \subsection{Proof of Proposition \ref{functional}}
   \subsubsection{Pseudo-spectral projection }
   The detailed definition of the pseudo-spectral projection is provided in \cite{maday1982spectral, canuto1982approximation}. For readability, we restate the definition here.

First, we define the Fourier basis of \( C(\Omega) \) as
\[
\phi_{\vk}(\vx) = \exp(\I 2\pi \vk \cdot \vx),
\]
where \( \vk \in \mathbb{Z}^d \). We then define a bilinear form in \( C(\Omega) \) based on the grid points
\[
\left\{ \vx_{\vnu} \mid \vx_{\vnu} = \frac{\vnu}{2N+1}, \vnu \in \{0,1,\ldots,2N\}^d \right\},
\]
given by
\begin{equation}
    (f, g)_N = \frac{1}{(2N+1)^d} \sum_{\vnu \in \{0,1,\ldots,2N\}^d} f(\vx_{\vnu}) \overline{g(\vx_{\vnu})}.\notag
\end{equation}

The pseudo-spectral projection \( \fP_c \) is defined as follows for any \( f \in C(\Omega) \):
\[
\fP_c f = \sum_{|\vk|_\infty \leq N} (f, \phi_{\vk})_N \phi_{\vk} = \frac{1}{(2N+1)^d} \sum_{|\vk|_\infty \leq N} \sum_{\vnu \in \{0,1,\ldots,2N\}^d} f(\vx_{\vnu}) \overline{\phi_{\vk}(\vx_{\vnu})} \phi_{\vk}.
\]

Therefore, the pseudo-spectral projection \( \fP_c \) can be divided into two parts. First, we define
\begin{equation}\label{D}
    \fD(f) = (f(\vx_{\vnu}))_{\vnu \in \{0,1,\ldots,2N\}^d}.
\end{equation}
Then, we define \( \fP : [-M, M]^m \to C(\sT^d) \) such that
\begin{equation}\label{P}
    \fP[(f(\vx_{\vnu}))_{\vnu \in \{0,1,\ldots,2N\}^d}] := \frac{1}{(2N+1)^d} \sum_{|\vk|_\infty \leq N} \sum_{\vnu \in \{0,1,\ldots,2N\}^d} f(\vx_{\vnu}) \overline{\phi_{\vk}(\vx_{\vnu})} \phi_{\vk},
\end{equation}
where \( M = \|f\|_{L^\infty(\Omega)} \).

The approximation estimate of \( \fP_c \) is shown by the following lemma:\begin{lemma}[\cite{canuto1982approximation}]\label{pc}
    Let $d \in \mathbb{N}$. For any $s>d / 2$ and $N \in \mathbb{N}$, the pseudo-spectral projection $\fP_c: H^s\left(\mathbb{T}^d\right) \rightarrow C\left(\mathbb{T}^d\right)$ is well-defined. Furthermore, there exists a constant $C=C(s, d)>0$, such that the following approximation error estimate holds
$$
\left\|f-\fP_c f \right\|_{H^{\varsigma}(\Omega)} \leq C N^{-(s-\varsigma)}\|f\|_{H^s(\Omega)}, \quad \forall f \in H^s\left(\mathbb{T}^d\right)
$$

for any $\varsigma \in[0, s]$.
\end{lemma}

\subsubsection{Proof of Proposition \ref{functional}}
\begin{proof}
   By defining \( \fD \) and \( \fP \) as in \eqref{D} and \eqref{P}, and using Lemma \ref{pc}, we obtain the following:
\begin{equation}
    \left\| f - \fP \circ \fD(f) \right\|_{H^{s'}(\Omega)} \leq C N^{s' - s} \|f\|_{H^s(\Omega)}, \quad \forall f \in H^s\left(\mathbb{T}^d\right),\notag
\end{equation}
where \( C \) is a constant independent of \( N \). Given that \( s' > n + \frac{d}{2} \), by Sobolev embedding, we have:
\begin{equation}
    \left\| f - \fP \circ \fD(f) \right\|_{W^{n,\infty}(\Omega)} \leq C \left\| f - \fP \circ \fD(f) \right\|_{H^{s'}(\Omega)}.\notag
\end{equation}
Thus, we conclude:
\begin{equation}
    \left\| f - \fP \circ \fD(f) \right\|_{W^{n,\infty}(\Omega)} \leq C N^{s' - s} \|f\|_{H^s(\Omega)}, \quad \forall f \in H^s\left(\mathbb{T}^d\right),\notag
\end{equation}
where \( C \) is a constant independent of \( N \). Next, we show that \( \fP \) is a Lipschitz continuous map \( [-M, M]^m \to H^{s'}(\sT^d) \). For any \( f_1, f_2 \in H^{s'}(\sT^d) \), we have:
\begin{align}
    \|\fP\circ\fD f_1-\fP\circ\fD f_2\|^2_{H^{s'}(\Omega)}\le&\|\fP_c f_1-\fP_c f_2\|^2_{H^{s'}(\Omega)}\notag\\\le& \|\fP_c (f_1- f_2)\|^2_{H^{s'}(\Omega)}=\left\|\sum_{|\vk|_\infty \leq N} (f_1-f_2, \phi_{\vk})_N \phi_{\vk}\right\|^2_{H^{s'}(\Omega)}\notag\\\le& \sum_{|\vk|_\infty \leq N}|\vk|^{2s'}(f_1-f_2, \phi_{\vk})_N^2\notag\\\le & \sum_{|\vk|_\infty \leq N}\frac{1}{(2N+1)^d}|\vk|^{2s'}|D(f_1)-D(f_2)|^2\notag
\end{align}
where the last inequality follows from the Cauchy–Schwarz inequality. Therefore, we have:
\begin{align}
    \|\fP \vz_1 - \fP \vz_2\|_{H^{s'}(\Omega)} \leq C(N, d) |D(f_1) - D(f_2)|,\notag
\end{align}
where
\begin{align}
    C(N,d)\le& \sqrt{\sum_{|\vk|_\infty \leq N}\frac{1}{(2N+1)^d}|\vk|^{2s'}}\le C\sqrt{\frac{1}{(2N+1)^d}\int_{{|\vk|_\infty \leq N}}|\vk|^{2s'}\mathrm{d} \vk}\notag\\\le& C\sqrt{\frac{ N^{d+2 s^{\prime}}}{(2 N+1)^d\left(d+2 s^{\prime}\right)}}\le CN^{s'}.\notag
\end{align}
\end{proof}

\subsection{Proof of Theorem \ref{approximation}}
\begin{proof}
    Based on Proposition \ref{operator}, we know that 
    \begin{align}
        \sup_{f \in \fX} \left\| \fG_*(f) - \sum_{k=1}^{p} c_k(\fG_*(f)) \fT(\vy; \vtheta_{2,k}) \right\|_{H^{2}(\Omega)}
        \leq C p^{-\frac{n-2}{d}} \sup_{f \in \fX} \|\fG_*(f)\|_{W^{n,\infty}(\Omega)} \leq C p^{-\frac{n-2}{d}} L M,\notag
    \end{align}
    where \( L \) is the Lipschitz constant of \( \fG_* \), and \( M \) bounds \( \|f\|_{W^{n,\infty}(\Omega)} \).

    Based on Proposition \ref{functional}, we also have:
    \begin{align}
        &\sup_{f \in \fX} \left\| \sum_{k=1}^{p} c_k(\fG_*(f)) \fT(\vy; \vtheta_{2,k}) - \sum_{k=1}^{p} c_k(\fG_*(\fP \circ \fD(f))) \fT(\vy; \vtheta_{2,k}) \right\|_{H^{2}(\Omega)} \notag \\
        \leq & \sum_{k=1}^{p} \sup_{f \in \fX} \left\| c_k(\fG_*(f)) \fT(\vy; \vtheta_{2,k}) - c_k(\fG_*(\fP \circ \fD(f))) \fT(\vy; \vtheta_{2,k}) \right\|_{H^{2}(\Omega)} \notag \\
        \leq & \sum_{k=1}^{p} \sup_{f \in \fX} \left\| \fT(\vy; \vtheta_{2,k}) \right\|_{H^{2}(\Omega)} \sup_{f \in \fX} |c_k(\fG_*(f)) - c_k(\fG_*(\fP \circ \fD(f)))| \notag \\
        \leq & C N^{s' - s} \sup_{f \in \fX} \|f\|_{H^s(\Omega)} \sum_{k=1}^{p} \left\| \fT(\vy; \vtheta_{2,k}) \right\|_{H^{2}(\Omega)},\notag
    \end{align}
    for \( s' \in \left(n + \frac{d}{2}, s\right) \). Since \( \sup_{f \in \fX} \|f\|_{H^s(\Omega)} \leq M \) and \( \fT(\vy; \vtheta_{2,k}) = \vx^\alpha s_{\vm}(\vx) \) for some \( |\valpha| \leq n-1 \) and \( \vm \in [\lceil p^{\frac{1}{d}} \rceil]^d \), we have:
    \begin{align}
        \left\| \fT(\vy; \vtheta_{2,k}) \right\|_{H^{2}(\Omega)} \leq \left\| s_{\vm}(\vx) \right\|_{H^{2}(\Omega)} \leq C p^{-1} p^{\frac{2}{d}}.\notag
    \end{align}
    Thus, we obtain:
    \begin{align}
        \sup_{f \in \fX} \left\| \sum_{k=1}^{p} c_k(\fG_*(f)) \fT(\vy; \vtheta_{2,k}) - \sum_{k=1}^{p} c_k(\fG_*(\fP \circ \fD(f))) \fT(\vy; \vtheta_{2,k}) \right\|_{H^{2}(\Omega)}
        \leq C N^{s' - s} M p^{\frac{2}{d}}.\notag
    \end{align}

    Similarly, based on Proposition 3, we have:
    \begin{align}
        \sup_{f \in \fX} \left\| \sum_{k=1}^{p} c_k(\fG_*(f)) \fT(\vy; \vtheta_{2,k}) - \sum_{k=1}^{p} \fB(\fD f; \vtheta_{1,k}) \fT(\vy; \vtheta_{2,k}) \right\|_{H^{2}(\Omega)}
        \leq C_2 m^{\frac{s'}{d}}q^{-\frac{\lambda}{m}} M p^{\frac{2}{d}},\notag
    \end{align}
    where \( C_2 \) is independent of \(m, p, q \).

    Finally, by combining the three parts, we obtain the desired result.
\end{proof}

\section{Proofs of Generalization Error }
The this part we need to do is bounded $\mathbb{E}_{\fY \sim \operatorname{Unif}(\Omega)^P} \mathbf{R}_{\fY}(\fC_{f}),\mathbb{E}_{\fX^M\sim\mu_{\fX}^M} \mathbf{R}_{\fY}(\fC_{f})$. Those terms can be bounded by the uniform covering number, which is 
		\begin{definition}[Uniform covering number \cite{anthony1999neural}]
				Let $(V,\|\cdot\|)$ be a normed space, and $\Theta_*\subset V$. $\{V_1,V_2,\ldots,V_a\}$ is an $\varepsilon$-covering of $\Theta_*$ if $\Theta_*\subset \cup_{i=1}^aB_{\varepsilon,\|\cdot\|}(V_i)$. The \textit{covering number }$\fN(\varepsilon,\Theta_*,\|\cdot\|)$ is defined as \[\fN(\varepsilon,\Theta_*,\|\cdot\|):=\min \{a: \exists \varepsilon \text {-covering over } \Theta_* \text { of size } a\} \text {. }\]Suppose the $\fF$ is a class of functions from $\fF$ to $\sR$. Given $a$ samples $\vZ_a=(z_1,\ldots,z_a)\in\fZ^a$, define \[\fF|_{\vZ_a}=\{(u(z_1),\ldots,u(z_a)):u\in\fF\}.\]The \textit{uniform covering number} $\fN(\varepsilon,\fF,a)$ is defined as \[\fN(\varepsilon,\fF,a)=\max_{\vZ_a\in\fX^a}\fN\left(\varepsilon, \fF|_{\vZ_a},\|\cdot\|_{\infty}\right),\]where $\fN\left(\varepsilon, \fF|_{\vZ_a},\|\cdot\|_{\infty}\right)$ denotes the $\varepsilon$-covering number of $\fF|_{\vZ_a}$ w.r.t the $l^\infty$-norm defined as $\|f\|_{\infty}=\sup_{\vz_i\in\vZ_a}|f(z_i)|$.
			\end{definition}

            Then we use a lemma to estimate the Rademacher complexity using the covering number.
			
			\begin{lemma}[Dudley's theorem \cite{anthony1999neural}]\label{dudley}
				Let $\fF$ be a function class such that $\sup_{g\in\fF}\|g\|_\infty\le B$. Then the Rademacher complexity $\mathbf{R}_S(\fF)$ satisfies that \[\mathbb{E}_{S\sim\rho^m} \mathbf{R}_S(\fF) \leq \inf _{0 \leq \delta \leq B}\left\{4 \delta+\frac{12}{\sqrt{a}} \int_\delta^B \sqrt{\log 2\fN(\varepsilon,\fF,a)} \,\mathrm{d} \varepsilon\right\}\]
			\end{lemma}
			
			To bound the Rademacher complexity, we employ Lemma \ref{dudley}, which bounds it by the uniform covering number. We estimate the uniform covering number by the pseudo-dimension based on the following lemma.

			\begin{lemma}[{\cite[Theorem 12.2]{anthony1999neural}}]\label{cover dim}
				Let $\fF$ be a class of functions from $\fH$ to $[-B,B]$. For any $\varepsilon>0$, we have \[\fN(\varepsilon,\fF,n)\le \left(\frac{2eaB}{\varepsilon\text{Pdim}(\fF)}\right)^{\text{Pdim}(\fF)}\] for $a\ge \text{Pdim}(\fF)$.
			\end{lemma}

            Finally, we just need to bound $\text{Pdim}(\fC_f)$ for any fixed $f \in \fX$ and $\text{Pdim}(\fJ_{\vy})$ for any fixed $\vy \in \Omega$. In the proof, we will use following lemmas:\begin{lemma}[{\cite[Lemma 17]{bartlett2019nearly},\cite[Theorem 8.3]{anthony1999neural}}]\label{bounded}
			Suppose $W\le M$ and let $P_1,\ldots,P_M$ be polynomials of degree at most $D$ in $W$ variables. Define \[K:=\left|\{\left(\sgn(P_1(\vz)),\ldots,\sgn(P_M(\vz))\right):\vz\in\sR^W\}\right|,\] then we have $K\le 2(2eMD/W)^W$.
		\end{lemma}
        \begin{lemma}[{\cite[Lemma 18]{bartlett2019nearly}}]\label{inequality}
			Suppose that $2^m\le 2^t(mr/w)^w$ for some $r\ge 16$ and $m\ge w\ge t\ge0$. Then, $m\le t+w\log_2(2r\log_2r)$.
		\end{lemma}

\begin{proof}[Proof of Proposition \ref{boundvc}]
    Without loss of generality, we present the method to bound $\text{Pdim}(\mathcal{C}_f)$ here. The method for bounding $\text{Pdim}(\mathcal{J}_{\boldsymbol{y}})$ will follow a similar approach. In order to bound pseudo-dimension, we introduce VC dimension (\cite{abu1989vapnik}) in the first: Let $H$ denote a class of functions from $\fX$ to $\{0,1\}$. For any non-negative integer $m$, define the growth function of $H$ as \[\Pi_H(m):=\max_{x_1,x_2,\ldots,x_m\in \fX}\left|\{\left(h(x_1),h(x_2),\ldots,h(x_m)\right): h\in H \}\right|.\] The Vapnik--Chervonenkis dimension (VC-dimension) of $H$, denoted by $\text{VCdim}(H)$, is the largest $m$ such that $\Pi_H(m)=2^m$. For a class $\fG$ of real-valued functions, define $\text{VCdim}(\fG):=\text{VCdim}(\sgn(\fG))$, where $\sgn(\fG):=\{\sgn(f):f\in\fG\}$ and $\sgn(x)=1[x>0]$.

    Then denote $\fD_f:=\{\eta(\vy,y_{d+1}):\eta(\vy,y_{d+1})=\psi(\vy)-y_{d+1},\psi\in \fC_f\}.$
			Based on the definition of VC-dimension and pseudo-dimension, we have that\begin{equation}
				\text{Pdim}(\fC_f)\le \text{VCdim}(\fD_f).\notag
			\end{equation} Then we just need to bound $\text{VCdim}(\fD_f)$. Denote $\bar{\vy}=(\vy,y_{d+1})\in \sR^{d+1}$. Let $\bar{\vy}$ be an input and $\vtheta\in\sR^W$ be a parameter vector in $\eta\in\fD_f$. We denote the output of $\eta$ with input $\bar{\vy}$ and parameter vector $\vtheta$ as $g(\bar{\vy},\vtheta)$. For fixed $\bar{\vy}_1,\bar{\vy}_2,\ldots,\bar{\vy}_m$ in $\sR^{d+1}$, we aim to bound\begin{align}
				K:=\left|\{\left(\sgn(g(\bar{\vy}_1,\vtheta)),\ldots,\sgn(g(\bar{\vy}_m,\vtheta))\right):\vtheta\in\sR^W\}\right|.\notag
			\end{align}
			
			For any partition $\fS=\{P_1,P_2,\ldots,P_T\}$ of the parameter domain $\sR^W$, we have \[K\le \sum_{i=1}^T\left|\{\left(\sgn(g(\bar{\vy}_1,\vtheta)),\ldots,\sgn(g(\bar{\vy}_m,\vtheta))\right):\vtheta\in P_i\}\right|.\] We choose the partition such that within each region $P_i$, the functions $g(\bar{\vy}_j,\cdot)$ are all fixed polynomials of bounded degree. This allows us to bound each term in the sum using Lemma \ref{bounded}.

           Due to Assumption \ref{first lip}, we know that $g(\bar{\vy}_j, \vtheta)$ is expressed as  
\[
F\left(\fG(f; \vtheta)(\vy), (\partial_i \fG(f; \vtheta)(\vy))_{i=1}^d, (\partial_{ij} \fG(f; \vtheta)(\vy))_{i,j=1}^d\right) - f(\vy) - y_{d+1},
\]
where $F$ is a polynomial function of degree $d_{\fL}$.

Therefore, it suffices to find a partition $\fS = \{P_1, P_2, \ldots, P_T\}$ of the parameter domain $\mathbb{R}^W$ such that  
\[
\fG(f; \vtheta)(\vy), \quad (\partial_i \fG(f; \vtheta)(\vy))_{i=1}^d, \quad \text{and} \quad (\partial_{ij} \fG(f; \vtheta)(\vy))_{i,j=1}^d
\]
are all fixed polynomials of bounded degree for fixed $\vy$. Under this partition, $g(\bar{\vy}_j, \cdot)$ will also be fixed polynomials of bounded degree.

We consider $\partial_{11} \fG(f; \vtheta)(\vy)$ here; other terms can be considered in a similar way, since 
\[
\partial_{11} \fG(f; \vtheta)(\vy) = \sum_{k=1}^p \fB(\fD(f); \vtheta_{1,k}) \partial_{11} \fT(\vy; \vtheta_{2,k}).
\]
Based on \cite[Theorem 3]{yang2023nearlys}, we know that we can divide $\mathbb{R}^W$ into 
\[
\prod_{n=1}^{L_T} 2 \left(\frac{2em(1+(n-1)2^{n-1})W_T}{\sum_{i=1}^n(W_T^2 + W_T)}\right)^{\sum_{i=1}^n W_T^2 + W_T}
\]
parts to make $\partial_{11} \fT(\vy; \vtheta_{2,k})$ a polynomial function in $\sum_{i=1}^{L_T+1} W_T^2 + W_T$ variables, with total degree no more than 
\[
d_1 := 1 + L_T 2^{L_T}
\]
on each partition region.

Based on \cite[Theorem 6]{bartlett2019nearly}, we know that we can divide $\mathbb{R}^W$ into 
\[
\prod_{n=1}^{L_B} 2 \left(\frac{2emnW_B}{\sum_{i=1}^n(W_B^2 + W_B)}\right)^{\sum_{i=1}^n W_B^2 + W_B}
\]
parts to make $\fB(\fD(f); \vtheta_{1,k})$ a polynomial function in $\sum_{i=1}^{L_B+1} W_B^2 + W_B$ variables, with total degree no more than 
\[
d_2 := L_B + 1
\]
on each partition region.

{\color{black}The domain–decomposition argument proceeds layer by layer.  
Assume inductively that, on every subregion of the current partition, the input to the layer is a polynomial.  
Lemma~\ref{bounded} bounds the number of sign changes of each such polynomial; we therefore subdivide each subregion at these sign–change points.  
Because the ReLU activation (and its square) is piecewise linear, its output is still a polynomial on every newly created subregion where the input keeps a fixed sign.  
Refining the partition in this manner for all neurons in the layer yields a new decomposition on which the layer’s output remains polynomial.  
Repeating this procedure for every layer produces the desired global domain decomposition for the entire network.
The more details can be found in \cite{yang2023nearlys,bartlett2019nearly}.}

Combining the above results, we can divide $\mathbb{R}^W$ into 
\[
\prod_{n=1}^{L_B} 2 \left(\frac{2emnW_B}{\sum_{i=1}^n(W_B^2 + W_B)}\right)^{\sum_{i=1}^n W_B^2 + W_B}
\cdot 
\prod_{n=1}^{L_T} 2 \left(\frac{2em(1+(n-1)2^{n-1})W_T}{\sum_{i=1}^n(W_T^2 + W_T)}\right)^{\sum_{i=1}^n W_T^2 + W_T}
\]
parts, to make $\fB(\fD(f); \vtheta_{1,k}) \partial_{11} \fT(\vy; \vtheta_{2,k})$ a polynomial function in 
\[
\sum_{i=1}^{L_B+1} \left(W_B^2 + W_B\right) + \sum_{i=1}^{L_T+1} \left(W_T^2 + W_T\right)
\]
variables, with total degree no more than $d_1 + d_2$ on each partition region.

Finally, we conclude that \(\sR^W\) can be divided into  
\[
\left(\prod_{n=1}^{L_B} 2 \left(\frac{2emnW_B}{\sum_{i=1}^n(W_B^2 + W_B)}\right)^{\sum_{i=1}^n W_B^2 + W_B} 
\cdot 
\prod_{n=1}^{L_T} 2 \left(\frac{2em(1+(n-1)2^{n-1})W_T}{\sum_{i=1}^n(W_T^2 + W_T)}\right)^{\sum_{i=1}^n W_T^2 + W_T}\right)^p
\]
parts, to make \(g(\bar{\vy}, \vtheta)\) a polynomial function in  
\[
p \sum_{i=1}^{L_B+1} \left(W_B^2 + W_B\right) + p \sum_{i=1}^{L_T+1} \left(W_T^2 + W_T\right)
\]
variables, with total degree no more than \(d_{\fL} + d_1 + d_2\) on each partition region.

Then based on Lemma~\ref{bounded}, for each partition \(\fS\), we have  
\begin{align}
    &\left|\left\{\left(\sgn(g(\bar{\vy}_1,\vtheta)),\ldots,\sgn(g(\bar{\vy}_m,\vtheta))\right):\vtheta\in\fS\right\}\right| \notag \\
    \le& 2\left(2em(d_{\fL} + d_1 + d_2)/p \sum_{i=1}^{L_B+1} \left(W_B^2 + W_B\right) + p \sum_{i=1}^{L_T+1} \left(W_T^2 + W_T\right)\right)^{p \sum_{i=1}^{L_B+1} \left(W_B^2 + W_B\right) + p \sum_{i=1}^{L_T+1} \left(W_T^2 + W_T\right)}.\notag
\end{align}

Then we have  
\begin{align}
    K \le & \left(\prod_{n=1}^{L_B} 2 \left(\frac{2emnW_B}{\sum_{i=1}^n(W_B^2 + W_B)}\right)^{\sum_{i=1}^n W_B^2 + W_B} 
    \cdot 
    \prod_{n=1}^{L_T} 2 \left(\frac{2em(1+(n-1)2^{n-1})W_T}{\sum_{i=1}^n(W_T^2 + W_T)}\right)^{\sum_{i=1}^n W_T^2 + W_T}\right)^p \notag \\
    &\cdot 2\left(\frac{2em(d_{\fL} + d_1 + d_2)}{p \sum_{i=1}^{L_B+1} \left(W_B^2 + W_B\right)} \right)^{p \sum_{i=1}^{L_B+1} \left(W_B^2 + W_B\right)}\cdot \left(\frac{2em(d_{\fL} + d_1 + d_2)}{p \sum_{i=1}^{L_T+1} \left(W_T^2 + W_T\right)}\right)^{p \sum_{i=1}^{L_T+1} \left(W_T^2 + W_T\right)}.\notag
\end{align}

We divide the above into two parts:
\begin{align}
    K_1 &= \left(\prod_{n=1}^{L_B+1} 2 \left(\frac{2emnW_B(n)}{\sum_{i=1}^n(W_B^2 + W_B)}\right)^{\sum_{i=1}^n W_B^2 + W_B}\right)^p, \notag\\
    K_2 &= \left(\prod_{n=1}^{L_T+1} 2 \left(\frac{2emnW_T(n)}{\sum_{i=1}^n(W_T^2 + W_T)}\right)^{\sum_{i=1}^n W_T^2 + W_T}\right)^p,\notag
\end{align}
where \(W_B(n)=W_B\) and \(W_T(n)=W_T\) for \(n\le L\), \(W_B(L_B+1)=\frac{d_{\fL} + d_1 + d_2}{L_B+1}\), and \(W_T(L_T+1)=\frac{d_{\fL} + d_1 + d_2}{2^{L_T}L_T+1}\).

By the AM-GM inequality, we have  
\begin{align}
    K_1 \le & \left( 2^{L_B+1} \left(\frac{\sum_{n=1}^{L_B+1}2emnW_B(n)}{\sum_{n=1}^{L_B+1}\sum_{i=1}^n W_B^2 + W_B}\right)^{\sum_{n=1}^{L_B+1}\sum_{i=1}^n W_B^2 + W_B}\right)^p \notag \\
    \le & \left( 2^{L_B+1} \left(\frac{em[(L_B+1)(L_B+2)W_B+2d_{\fL}+2d_1]}{\sum_{n=1}^{L_B+1}\sum_{i=1}^n W_B^2 + W_B}\right)^{\sum_{n=1}^{L_B+1}\sum_{i=1}^n W_B^2 + W_B}\right)^p.\notag
\end{align}
Similarly,  
\begin{align}
    K_2 \le \left( 2^{L_T+1} \left(\frac{2em[(L_T+3+(L_T-2)2^{L_T+2})W_T+d_{\fL}+d_1]}{\sum_{n=1}^{L_T+1}\sum_{i=1}^n W_T^2 + W_T}\right)^{\sum_{n=1}^{L_T+1}\sum_{i=1}^n W_T^2 + W_T}\right)^p.\notag
\end{align}

Finally, applying the AM-GM inequality again, we obtain  
\begin{align}
    &K \le 2^{p(L_T+L_B+2)}\notag\\&\cdot\left(\frac{em[2(L_T+3+(L_T-2)2^{L_T+2})W_T+(L_B+1)(L_B+2)W_B+4d_{\fL}+2d_1+2d_2]}{U}\right)^{pU}.\notag
\end{align}

By the claim proven in \cite{bartlett2019nearly}:  \textit{Suppose that \(2^m\le 2^t(mr/w)^w\) for some \(r\ge 16\) and \(m\ge w\ge t\ge0\). Then \(m\le t+w\log_2(2r\log_2r)\). }

We have  
\begin{align}
    &2^{\text{VCdim}(\fD_f)} 
    \le  2^{p(L_T+L_B+2)}\notag\\&\cdot\left(\frac{ep\text{VCdim}(\fD_f)[2(L_T+3+(L_T-2)2^{L_T+2})W_T+(L_B+1)(L_B+2)W_B+4d_{\fL}+2d_1+2d_2]}{pU}\right)^{pU}.\notag
\end{align}
Thus,  
\begin{align}
    \text{VCdim}(\fD_f) \le Cp(L_T^3W_T^2+L_B^2W_B^2)\log_2(L_TL_BW_BW_T)\notag
\end{align}
due to Lemma \ref{inequality}, where \(C\) is a constant that depends logarithmically on \(d_{\fL}\).
\end{proof}

\begin{proof}[Proof of Theorem \ref{genn}]
    Due to Lemma \ref{first lemma}, we have \begin{equation}
        \mathbb{E} \sup_{\vtheta\in\Theta}\left[ L_M(\vtheta) - L_S(\vtheta) \right] \leq \left(\mathbb{E}_{f \sim \mu_{\fX}} \left[ 4 \Psi(f) \right]^2\right)^{\frac{1}{2}}\left(\mathbb{E}_{f \sim \mu_{\fX}}\left[\mathbb{E}_{\fY \sim \operatorname{Unif}(\Omega)^P} \mathbf{R}_{\fY}(\fC_{f}) \right]^2\right)^{\frac{1}{2}}.\notag
    \end{equation} For $\mathbb{E}_{\fY \sim \operatorname{Unif}(\Omega)^P} \mathbf{R}_{\fY}(\fC_{f})$, based on Lemmas \ref{dudley} and \ref{cover dim}, we have that \begin{align}
        \mathbb{E}_{\fY \sim \operatorname{Unif}(\Omega)^P} \mathbf{R}_{\fY}(\fC_{f})\le& 4\delta+\frac{12}{\sqrt{P}} \int_\delta^{\Psi(f)} \sqrt{\log 2\fN(\varepsilon,\fC_{f},P)} \,\mathrm{d} \varepsilon\notag\\\le &4\delta+\frac{12}{\sqrt{P}} \int_\delta^{\Psi(f)} \sqrt{\log 2\left(\frac{2eP\Psi(f)}{\varepsilon\text{Pdim}(\fC_{f})}\right)^{\text{Pdim}(\fC_{f})}} \,\mathrm{d} \varepsilon\notag\\\le&4\delta+\frac{12\Psi(f)}{\sqrt{P}}+12\left(\frac{\text{Pdim}(\fC_{f})}{P}\right)^{\frac{1}{2}} \int_\delta^{\Psi(f)} \sqrt{\log \left(\frac{2eP\Psi(f)}{\varepsilon\text{Pdim}(\fC_{f})}\right)} \,\mathrm{d} \varepsilon.\notag
				\end{align} By the direct calculation for the integral, we have\[\int_\delta^{\Psi(f)} \sqrt{\log \left(\frac{2eP\Psi(f)}{\varepsilon\text{Pdim}(\fC_{f})}\right)} \,\mathrm{d} \varepsilon\le \Psi(f)\sqrt{\log \left(\frac{2eP\Psi(f)}{\delta\text{Pdim}(\fC_{f})}\right)}.\]
				
				Then choosing $\delta=\Psi(f)\left(\frac{\text{Pdim}(\fC_{f})}{P}\right)^{\frac{1}{2}}\le \Psi(f)$, we have \begin{equation}
					 \mathbb{E}_{\fY \sim \operatorname{Unif}(\Omega)^P} \mathbf{R}_{\fY}(\fC_{f})\le 28 \Psi(f)\left(\frac{\text{Pdim}(\fC_{f})}{P}\right)^{\frac{1}{2}}\sqrt{\log \left(\frac{2eP}{\text{Pdim}(\fC_{f})}\right)}.\notag
				\end{equation}Therefore, due to Proposition \ref{boundvc}, \begin{equation}
        \mathbb{E}_{\fY \sim \operatorname{Unif}(\Omega)^P} \mathbf{R}_{\fY}(\fC_{f})\le C \Psi(f) \frac{\left(p(L_T^3W_T^2+L_B^2W_B^2)\log_2(L_TL_BW_BW_T)\right)^{\frac{1}{2}}}{\sqrt{P}}\log P.\notag
    \end{equation} Due to Assumption \ref{first lip}, we have \begin{equation}
        \mathbb{E} \sup_{\vtheta\in\Theta}\left[ L_M(\vtheta) - L_S(\vtheta) \right]\le C G^2 \frac{\left(p(L_T^3W_T^2+L_B^2W_B^2)\log_2(L_TL_BW_BW_T)\right)^{\frac{1}{2}}}{\sqrt{P}}\log P.\notag
    \end{equation} The estimation of $\mathbb{E} \sup_{\boldsymbol{\theta} \in \Theta} \left[ L_D(\boldsymbol{\theta}) - L_M(\boldsymbol{\theta}) \right]$ is similar and is omitted here for brevity.
\end{proof}

 \bibliographystyle{elsarticle-harv} 
 \bibliography{references}
\end{document}